%% file: main.tex
\documentclass{article}

\usepackage[preprint]{neurips_2021}

\usepackage{amsmath}
\usepackage[utf8]{inputenc} 
\usepackage[T1]{fontenc}    
\usepackage{hyperref}       
\usepackage{url}            
\usepackage{booktabs}       
\usepackage{amsfonts}       
\usepackage{nicefrac}       
\usepackage{microtype}      
\usepackage{xcolor}         
\usepackage{graphicx}
\usepackage{float}
\usepackage{amsthm}
\usepackage{cleveref}
\usepackage{natbib}
\usepackage{bm}

\newtheorem{theorem}{Theorem}

\newtheorem{lemma}{Lemma}

\newtheorem{example}{Example}

\theoremstyle{definition}

\theoremstyle{remark}

\input{notations}

\hypersetup{
    colorlinks = true,
 	allcolors = teal,
}

\title{The Optimal Size of an Epistemic Congress}

\author{%
  Manon Revel \\
  MIT\\
  \texttt{mrevel@mit.edu} \\
  \And
  Tao Lin \\
  Harvard University \\ 
  \texttt{tlin@g.harvard.edu} \\
  \And
  Daniel Halpern \\
  Harvard University \\
  \texttt{dhalpern@g.harvard.edu}
}

\begin{document}
\sloppy

\maketitle

\begin{quote}
    \textit{However small the Republic may be, the Representatives must be raised to a certain number, in order to guard against the cabals of a few; and however large it may be, they must be divided to certain number, in order to guard against the confusion of a multitude. (Federalist Paper No.~10)}
    \hfill -- James Madison
\end{quote}

\begin{abstract}
  We analyze the optimal size of a congress in a representative democracy. We take an \emph{epistemic} view where voters decide on a binary issue with one ground truth outcome, and each voter votes correctly according to their competence levels in $[0, 1]$. Assuming that we can sample the best experts to form an \emph{epistemic congress}, we find that the optimal congress size should be linear in the population size. This result is striking because it holds even when allowing the top representatives to be accurate with arbitrarily high probabilities. We then analyze real world data, finding that the actual sizes of congresses are much smaller than the optimal size our theoretical results suggest. We conclude by analyzing under what conditions congresses of sub-optimal sizes would still outperform direct democracy, in which all voters vote.
\end{abstract}

\section{Introduction}
Modern governments often take the form of a representative democracy, that is, a college of chosen representatives form a congress to make decisions on behalf of the citizenry. Clearly, the performance of the congress depends the number of representatives, and the optimal number of representatives has been subject to great debates (see activists at \url{https://thirty-thousand.org} who advocate for enlarging the congress). In the Federalist Paper No.~56, Madison argues that there shall be \textit{a representative for every thirty thousand inhabitants} and the American congress was actually enlarged every ten years between 1785 and 1913 from 65 to 435, adapting the evolution of the States' population~\citep{szpiro2010numbers}, and remained constant since 1913. 

Quantitative research aiming at rationalizing the optimal congress size dates back to the 1970s. 
\citet{taagepera1972size} concluded that the number of representatives should be the cube-root of the population size. 
These findings are regarded as seminal~\citep{jacobs2015explaining} and have influenced political decisions and referendums, such as the 2020 Italian referendum to reduce the size of both chambers from 945 to 600 parliamentary~\citep{margaritondo2021size, 945}.

Yet, recent work using machinery from physics and economics revisited these claims and showed that, under different assumptions, the optimal number should be larger, at least proportional to the square-root of the population size~\citep[e.g.,][]{auriol2012optimal, margaritondo2021size}. In particular, \citet{magdon-ismail_mathematical_2018} explored an \emph{epistemic} set-up where voters are grouped into pods of size $L$, and one representative is selected from each pod. The authors find that the congress size ought to be linear under this model when voting is cost-less. Adding that the cost of the congress is polynomial in the number of representatives, and the benefit from finding the ground truth is polynomial in the number of voters, the optimal congress size decreases to $O(\log n)$.

Finally, as observed by \citet{magdon-ismail_mathematical_2018}, a congress in the real world resembles an ensemble of classifiers in machine learning: classifiers are ``voters'' who predict a binary value.
To obtain a good ensemble of classifiers, one can measure the accuracy of all classifiers and keep only the most accurate ones.  A key question then is: how many classifiers should we keep? 

\subsection{Our Contribution}
Through novel proofs techniques, we strengthen the pessimistic results of \citet{magdon-ismail_mathematical_2018} for congress under the epistemic approach, finding that even with the ability to identify the most accurate members of society to form a congress, the optimal congress size remains linear in the size of population size. However, we find that all is not lost for congresses of more practical sizes. We follow this up with comparisons of different sizes and identify conditions for smaller congresses to be more accurate than when the entire society votes. 

In the epistemic setting, voters decide on a binary issue and aim at differentiating between the ground truth correct choice, the value $1$, and its alternative, $0$. Each voter has a competence level in $[0, 1]$ representing the probability that the voter votes correctly. Further, the competence levels of the population are drawn according to some distribution. We take the idealized view that given a target size $k$, we can identify the $k$ most competent voters in society to form the congress, who then vote on the issue following the majority opinion. We conclude that, should voters' competence levels be the expected values of the order statistics from uniform distribution $\mathcal{U}(0, 1)$, the optimal size of congress is between $(3 - 2\sqrt{2})n$ and $\frac{n}{2}$. For arbitrary distributions where the maximum competence level is bounded away from $1$ and the inverse cumulative distribution function is Lipschitz continuous, the optimal size is $\Theta(n)$ with more refined bounds depending on the distribution.

We then turn to studying real-world data on the sizes of countries' representative bodies. Here, we notice that congresses in the real world are of order cube-root of the population size, hence much smaller than the optimal size (linear) our theoretical results suggest. We then find under what conditions on the distribution of competence level a smaller congress still outperforms the majority. If the population is unbiased or biased towards 0, a congress composed of experts with expertise level higher that $0.5$ trivially outperforms the majority. We further find that, for a population whose average level of competence is biased above $0.5$,
a relatively small congress can still be better than the majority as long as the bias is small enough, and worse when the bias is large. 
We characterize this threshold for both one-person and $n^r$-person congresses.  

\subsection{Related Work}
\label{sec:related}
The use of an epistemic approach, using voting to aggregate objective opinions, is well studied in computational social choice~\citep{brandt_handbook_2016}. One particularly important result is known as the Condorcet Jury Theorem~\citep{condorcet, grofman1983thirteen}, which shows that in the limit, a majority vote by an increasing number of independent voters biased towards the correct outcome will be correct with probability approaching 1. Subsequent work studied extensions of the Condorcet Jury Theorem in instances where the voters are inhomogenuous, dependent, or strategic, as summarized in a survey paper by \citet{nitzan2017collective}. 

The first work about the optimal size of parliaments focused on maximizing parliament's efficiency~\citep{taagepera1972size}. For them, maximizing efficiency was equivalent to minimizing the communication time spent on discussions with constituents --- the authors ultimately stated that the average time spent talking to the constituents per congress-members should be equal to the time spent talking to the other congress-members. Hence, \citet{taagepera1972size} argued that the optimal congress size should follow a ``cube-root law''. \citet{margaritondo2021size} revisited this work and found a flaw in the original proof, arguing that the optimal size under this model should in fact be $\Theta(\sqrt{n})$. Empirical papers~\citep{taagepera1972size, auriol2007more} that focused on finding the optimal number of representatives used country data to back up the ``square-root law'' result. \citet{jacobs2015explaining}, on the other hand, investigate potential causal effects of different congress sizes. 

The work of \citet{auriol2012optimal} also aims to derive the optimal number of representatives for a society.
However, their model lies in stark contrast to the epistemic one: they assume that voters have preference-based utilities, with an uninformative prior, and the representatives are chosen uniformly at random from society, while we take the best. They reach the conclusion that the optimal size of congress is proportional the square-root of the population size.
Further, \citet{zhao2020allometric} look at the optimal number of representatives as the minimum size of a node set such that all nodes in that set can reach other nodes in at most $m$ steps (where $m=\Theta(\log n)$ is an exogenous threshold). In this set up, they obtain an $O(n^\gamma)$ result with $\frac{1}{3}\leq\gamma\leq\frac{5}{9}$.

Finally, we build upon the work of \citet{magdon-ismail_mathematical_2018}. There, the authors consider a model for representative democracy where agents are grouped in $K$ groups of sizes $L$ and choose one representative per group. Importantly, the competences are drawn from a distribution $\dist$ after the agents are grouped. The authors then derive the group size that maximises the probability that the representatives make the correct decision. They show the optimal group size is constant, so the optimal number of representatives (which is, in the simplest set-up, the population size divided by the number of groups) should then be linear in the population size. The fact that the level of competence is drawn after grouping people imposes a trade-off between how accurate the representatives will be and how many representatives ($n/L$) there are. Indeed, the best agent in each group has competence level that is the top order statistic of the distribution with $L$ draws. With a uniform distribution, the top level of competence is of order $1 - \frac{1}{L+1}$, which is large if $L$ is large, i.e., when $n/L$ is small. The trade-off implied by the model is in favor of large congresses. The results of \citet{magdon-ismail_mathematical_2018} are pessimistic in that it is impractical to have congresses as big as a constant fraction of the population. One could wonder whether the optimal congress size remains linear if one allows the highest competences to become arbitrarily large. This is precisely the gap we fill. 

\section{Model}

Let $n$ be the number of voters in the society. Following the epistemic approach, voters need to choose between two options, $0$ and $1$, where $1$ is assumed to be the ground truth. Each voter $i$ is endowed with a level of expertise (or competence) $p_i \in [0,1]$, which is the probability that she votes ``correctly'' (i.e., votes for option $1$). Depending on the instance, we will sometimes assume that the $p_i$s are sampled from some distribution $\dist$ whose support is contained in $[0, 1]$ and other times assume the $p_i$s are deterministic (perhaps also depending on $n$ which will always be clear from context). 

Given $p_1, \ldots, p_n$, we sort voters by decreasing competence level, denoted by $p_{(1)} \ge \cdots \ge p_{(n)}$, where $p_{(i)}$ is the competence level of the $i^{\text{th}}$ best voter.\footnote{Note that, for notational convenience, this is the reverse of normal order statistics.}
Let $X_{(1)}, \ldots, X_{(n)}$ be Bernoulli random variables denoting their votes, with $X_{(i)}=1$ meaning a correct vote for the $i^{\text{th}}$ best voter and $0$ otherwise; the $X_{(i)}$s are conditionally independent given $p_{(i)}$s, and $\Pr[X_{(i)}=1\mid p_{(i)}] = p_{(i)}$.

A congress of size $k$ is composed of the $k$ best voters in society and makes a correct decision when a strict majority are correct, $\sum_{i=1}^k X_{(i)} > k/2$.\footnote{A strict rather than weak majority here corresponds to tie-breaking in favor of the incorrect outcome. Tie-breaking in the other direction would not asymptotically change our results.} 
One may envision other rules to select the congress members, for example the group representatives analyzed by \citet{magdon-ismail_mathematical_2018}. Here we take the best $k$ voters, and this can be seen as a best-case scenario for accuracy. Strikingly, as we will show, even under this strong assumption, the optimal number of representatives is already very large, which suggests that the optimal number would even be larger in more realistic scenarios.

\section{Optimal Congress Size}\label{sec:problem-uninformed}
In this section, we prove theoretical bounds on the optimal size of congress for several natural distributions. We begin by formally stating our problem.

For fixed voter competencies $p_{(1)} \ge \cdots \ge p_{(n)}$, we define $K^\star$ to be the optimal size of congress, the size $k$ that maximizes the probability that the representatives make a correct decision (for convenience breaking ties in favor of an arbitrary odd $k$\footnote{Note that there must always be an optimal $k$ that is odd, as for any even $k$, due to our strict majority constraint, $k - 1$ must have overall accuracy at least as high.}). Formally, 
\begin{equation*}
 K^\star \in \argmax_{1\le k\le n}\left\{ \Pr\left[\sum_{i=1}^k X_{(i)} > \frac{k}{2} ~\bigg|~ X_{(i)} \sim \mathrm{Bern}(p_{(i)})\right] \right\}.
\end{equation*}
We note that since $K^{\star}$ is a function of the voter competencies, if these competencies are random samples, then $K^{\star}$ is a random variable.
However, we sometimes assume for tractability that the competencies match their expectation, that is, $p_{(i)}$ is exactly equal to the expectation of the $(n + 1 - i)$'th order statistic of $n$ draws from $\dist$. In this case, $K^{\star}$ is a deterministic value for each $n$.

For fixed voter competencies $p_{(1)} \ge \cdots \ge p_{(n)}$, let $\mathcal{E}_{k}^j$ be the event that exactly $j$ of the top experts out of $n$ are correct.
Our characterization of the optimal size $K^\star$ relies on the following key lemma. 

\begin{lemma}
\label{lem:new-core}
For fixed competencies $p_{(1)} \ge \cdots \ge p_{(n)}$, for all odd $k \le n$ with $k = 2\ell + 1$,
\vspace{-0.7em}
\begin{itemize}
    \item If $\frac{\Pr[\mathcal{E}_{k}^{\ell + 1}]}{\Pr[\mathcal{E}_{k}^\ell]} < \frac{p_{(k+1)}p_{(k+2)}}{(1 - p_{(k+1)})(1-p_{(k+2)})}$, then $K^{\star} \ne k$.
\vspace{-0.5em}
    \item If $\frac{\Pr[\mathcal{E}_{k}^{\ell + 1}]}{\Pr[\mathcal{E}_{k}^\ell]} > \frac{p_{(k+1)}p_{(k+2)}}{(1 - p_{(k+1)})(1-p_{(k+2)})}$, then $K^{\star} \ne k + 2$.
\vspace{-0.4em}
\end{itemize}
\end{lemma}
The proof of the lemma involves comparing a congress of some specific size $k$ to one of size $k + 2$ (recall that chose $K^{\star}$ to be odd, so we may as well restrict ourselves to odd $k$). Clearly, if the top $k + 2$ experts have a higher chance of being correct than $k$, then $k$ cannot be optimal (and vice-versa). Importantly, this gives us a sufficient condition to rule out certain values of $k$. For example, if we know that for all $k < c$ the first condition of the lemma holds, then that implies $K^{\star} \ge c$.
\vspace{-0.5em}
\begin{proof}[Proof of \Cref{lem:new-core}]
    For any $k \le n$, let $q_k = \sum_{j = \lfloor q^k / 2 \rfloor + 1}^k \Pr[\mathcal{E}_k^j]$ be the probability that a congress of size $k$ will be correct. We have that $K^{\star} \in \argmax_{k \le n} q_k$.
    Fix $p_{(1)} \ge \cdots \ge p_{(n)}$ and a specific $k = 2\ell + 1$.
    We will show that $q_{k + 2} > q_k$ (resp. $<$) is equivalent to $\frac{\Pr[\mathcal{E}_{k}^{\ell + 1}]}{\Pr[\mathcal{E}_{k}^\ell]} < \frac{p_{(k+1)}p_{(k+2)}}{(1 - p_{(k+1)})(1-p_{(k+2)})}$ (resp. $>$). If $q_{k + 2} > q_k$ (resp. $<$), then $K^{\star} \ne k$ (resp. $k + 2$) as that would imply $K^{\star}$ is not optimal. 
    
    Let us now consider $q_{k + 2} - q_k$. The only way the two new experts can change the outcome from incorrect to correct is when exactly $\ell$ of the top $k$ experts were correct (so the majority of $k$ were incorrect), and the two new experts are correct. Conversely, the only scenario in which a correct outcome becomes incorrect is when exactly $\ell + 1$ of the top $k$ experts are correct while the two new experts are incorrect. Since $\mathcal{E}_{k}^j$ is the event that exactly $j$ of the top $k$ experts out of $n$ are correct, we can formally write the above as
	$$
    q_{k + 2} - q_k = \Pr[\mathcal{E}_{k}^\ell] \cdot p_{(k+1)}p_{(k+2)}  - \Pr[\mathcal{E}_{k}^{\ell + 1}] \cdot (1-p_{(k+1)})(1-p_{(k+2)}).
    $$
    Rearranging this yields the two equivalent inequalities previously stated.
\end{proof}
\vspace{-0em}

For a set of representatives $S \subseteq [k]$, let $w(S)=\prod_{i \in S} p_{(i)} \cdot \prod_{i\in [k] \setminus S} (1 - p_{(i)})$
	be the probability that exactly those in $S$ are correct (and those in $[k]\setminus S$ are incorrect). We then have the following.

\begin{lemma}\label{Recursion}
For each $\mathcal{E}_{k}^j$,
$
    \Pr[\mathcal{E}_{k}^j] = \frac{1}{k-j} \sum_{\substack{S \subseteq [k] \\ |S| = j+1}} w(S) \sum_{i \in S} \frac{1 - p_{(i)}}{p_{(i)}}.
$
\end{lemma}
\vspace{-1.5em}
\begin{proof}
By the definition of $\mathcal{E}_{k}^{j}$, 
$
	    \Pr[\mathcal{E}_{k}^{j}]
	    = \sum_{\substack{S \subseteq [k] \\ |S| = j}} w(S)
$. 
We then note that
	\begin{align*}
	    \sum_{\substack{S \subseteq [k] \\ |S| = j}} w(S) = \frac{1}{k-j} \sum_{\substack{S \subseteq [k] \\ |S| = j+1}} \sum_{i \in S} w(S \setminus \set{i})
	\end{align*}
because when we count the sets $S$ of size $j$ by first selecting a set of size $j+1$ and then removing one of its $j+1$ elements, each set of size $j$ is counted exactly $k-j$ times.
Therefore,
\[
	\Pr[\mathcal{E}_{k}^{j}] = \frac{1}{k-j} \sum_{\substack{S \subseteq [k] \\ |S| = j+1}} \sum_{i \in S} w(S \setminus \set{i}) = \frac{1}{k-j} \sum_{\substack{S \subseteq [k] \\ |S| = j+1}} w(S) \sum_{i \in S} \frac{1 - p_{(i)}}{p_{(i)}}. 
\qedhere \]
\end{proof}

Armed with these lemmas, we can now move to proving bounds on the optimal congress size. 

\subsection{Standard Uniform Distribution}
First, we focus on the case where competence levels are drawn from uniform distribution $\mathcal{U}(0, 1)$. For tractability, as discussed in the problem statement, we assume that the competence levels are exactly equal to their expectation, i.e., $p_{(i)} = \frac{n + 1 - i}{n + 1}$ (see e.g., \citet{orderS}). 
In this case, the competence levels of the top experts approach to $1$ asymptotically. 
Strikingly, we find that even with top experts becoming arbitrarily accurate and with the ability to identify the most accurate members of society, the optimal size of congress still remains a constant fraction of the population.

\begin{theorem}
    \label{LinearBounds}
    Suppose $p_{(i)} = \frac{n + 1 - i}{n + 1}$. Then, $(3 - 2\sqrt{2}) \cdot n - O(1) \le K^{\star} \le \frac{1}{2} \cdot n + O(1).$
\end{theorem}

\begin{proof}
Recall that we can focus only on odd $k$. Fix some odd $k \le n$ where $k = 2\ell + 1$ for some non-negative integer $\ell$. Our goal will be to compare $\frac{\Pr[\mathcal{E}_{k}^{\ell + 1}]}{\Pr[\mathcal{E}_{k}^\ell]}$ and $ \frac{p_{(k+1)}p_{(k+2)}}{(1 - p_{(k+1)})(1-p_{(k+2)})} = \frac{(n-k)(n-k-1)}{(k+1)(k+2)}$ in order to apply \Cref{lem:new-core}.

By \Cref{Recursion} with $j = \ell$ and using the fact that $k - \ell = \ell + 1$,
\begin{equation}
    \label{eq:ejk}
    \Pr[\mathcal{E}_{k}^{\ell}]
	    = \frac{1}{\ell + 1} \sum_{\substack{S \subseteq [k] \\ |S| = \ell+1}} w(S) \sum_{i \in S} \frac{i}{n + 1 - i}.
\end{equation}

We begin with the lower bound. Let us consider the inner sum of \Cref{eq:ejk}. We have that for all $S$,
$$
\sum_{i\in S} \frac{i}{n+1-i} \ge \sum_{i\in S} \frac{i}{n} = \frac{1}{n} \sum_{i\in S} i \ge \frac{1}{n} \sum_{i=1}^{\ell+1} i = \frac{(\ell+1)(\ell+2)}{2n}
$$
where the first inequality holds because $i \ge 1$ for all $i$ and the second inequality holds because $|S| = \ell + 1$ and $S \subseteq [k]$ hence the minimum it could sum to is that of the smallest $\ell + 1$ positive integers. As this bound is independent of $S$, we can pull it out of the the outer sum to yield
$$
    \Pr[\mathcal{E}_{k}^{\ell}] \ge \frac{\ell + 2}{2n} \sum_{\substack{S \subseteq [k] \\ |S| = \ell+1}} w(S) = \frac{\ell + 2}{2n} \cdot \Pr[\mathcal{E}_{k}^{\ell + 1}] = \frac{k + 3}{4n} \cdot \Pr[\mathcal{E}_{k}^{\ell + 1}]
$$
where the last inequality holds because $\ell + 2 = \frac{k - 1}{2} + 2 = \frac{k + 3}{2}$. This allows us to write
$
    \frac{\Pr[\mathcal{E}_{k}^{\ell + 1}]}{\Pr[\mathcal{E}_{k}^\ell]} \le \frac{4n}{k+3},
$
so in order to invoke the first item of \Cref{lem:new-core} to show a certain value of $k$ is not optimal, we need a sufficient condition for $k$ to guarantee
\begin{equation}
\label{eq:suff-lower}
    \frac{4n}{k+3} < \frac{(n-k)(n-k-1)}{(k+1)(k+2)}.
\end{equation}
Note that \Cref{eq:suff-lower} is implied by
$4n < \frac{(n - k - 1)^2}{k + 1}$ which we can rearrange to $(k + 1)^2 - 6n(k + 1) + n^2 > 0$. The left hand side of the inequality is a quadratic in $(k + 1)$ with roots at $(3 \pm 2\sqrt{2}) \cdot n$. Since the squared term is positive and hence the quadratic is only non-positive between the two roots, as long as $(k + 1) < (3 - 2\sqrt{2}) \cdot n$, the inequality holds. Along with the first item of \Cref{lem:new-core}, this implies the desired $(3 - 2\sqrt{2}) \cdot n - O(1)$ lower bound.

Next, we will show the upper bound. In the inner summand of \Cref{eq:ejk}, $i \in [k]$ so $i \le k$, and hence $\frac{i}{n + 1 - i} \le \frac{k}{n + 1 - k}$. This yields
\begin{align*}
    \Pr[\mathcal{E}_{k}^{\ell}]
	    &\le \frac{1}{\ell + 1} \sum_{\substack{S \subseteq [k] \\ |S| = \ell+1}} w(S) \sum_{i \in S} \frac{k}{n + 1 - k}
	     \le \frac{1}{\ell + 1} \sum_{\substack{S \subseteq [k] \\ |S| = \ell+1}} w(S) \cdot |S|\cdot  \frac{k}{n + 1 - k}\\
	    &\qquad = \frac{k}{n + 1 - k}\sum_{\substack{S \subseteq [k] \\ |S| = \ell+1}} w(S) =  \frac{k}{n + 1 - k} \Pr[\mathcal{E}_{k}^{\ell + 1}].
\end{align*}
Here, we get that
$
\frac{\Pr[\mathcal{E}_{k}^{\ell + 1}]}{\Pr[\mathcal{E}_{k}^\ell]} \ge  \frac{k}{n + 1 - k}.
$
As with the lower bound, to invoke the second item of \Cref{lem:new-core}, we need a sufficient condition for
\begin{equation}
\label{eq:suff-upper}
    \frac{k}{n + 1 - k} > \frac{(n-k)(n-k-1)}{(k+1)(k+2)}.
\end{equation}
\Cref{eq:suff-upper} is equivalent to
$$k(k + 1)(k + 2) > (n - k - 1)(n - k)(n - k + 1).$$
As both sides are the product of three consecutive integers, this will be true as long as $n - k - 1 < k$, or equivalently $k + 2 > \frac{n}{2} + \frac{3}{2}$. Applying \Cref{lem:new-core} yields the desired upper bound.
\end{proof}

Hence, we have proved that for competencies equal to the expectation of $\mathcal{U}[0,1]$ order statistics, a constant fraction of the total population is necessary to maximize the probability the representatives make the correct decision. We conjecture that $K^\star$ is in fact close to $n/4$ in this set up (see simulations in \Cref{app:opt-congress-simulations}). 

\subsection{Distributions Bounded Away From 1}
Next, we consider a broad class of distributions which do not allow for arbitrarily accurate experts.  Unlike in the previous section, we do not fix $p_{(i)}$ to be their expectation; instead, they are random draws from $\mathcal D$.  Under relatively mild conditions, we show that the optimal size $K^\star$ grows linearly in the population size with high probability. 

\begin{theorem}
\label{thm:GeneralizationBoundL}
Let $\mathcal D$ be any continuous distribution supported by $[L, H]$ with cumulative distribution function $F(\cdot)$. If $0 < L < \frac{1}{2} < H < 1$, and $F^{-1}(\cdot)$ is $M$-Lipschitz continuous with $0< M <\infty$,\footnote{This condition is satisfied when the PDF of $\dist$ is lower bounded by $1/M$, which is satisfied by, e.g., uniform, normal, and beta distributions truncated to $[L, H]$ } then, with probability at least $1- 4e^{-2n\eps^2}$ the competency draws will yield an optimal $K^{\star}$ such that 
$$c_H n - O(1) \le K^\star \le c_L n + O(1)$$
for all $n$ and $\eps > 0$, where  $c_H = 1 - F\left(\frac{1}{1+\sqrt{\frac{1-H}{H}}} + M\eps\right)$ and $c_L = 1 - F\left(\frac{1}{1+\sqrt{\frac{1-L}{L}}} - M\eps\right)$.
\end{theorem}

We remark that $L\ge 0$ is sufficient for the lower bound $c_H n - O(1) \le K^\star$ to hold and vice-versa, $H\le 1$ is sufficient for the upper bound to hold. 
Both of these bounds individually hold with probability at least $1- 2e^{-2n\eps^2}$.

To prove \Cref{thm:GeneralizationBoundL}, we will make use of the following well-known concentration inequality. 
\begin{lemma}[Dvoretzky–Kiefer–Wolfowitz inequality, see e.g., \citealp{massart1990tight}]\label{lem:DKW1}
Let $p_{(1)} \ge \cdots \ge p_{(n)}$ be $n$ sorted i.i.d.~draws from $\mathcal D$.
For every $\varepsilon > 0$, 
$
    \Pr\left[ \forall i\in[n], \left| F(p_{(i)}) - \frac{n -i}{n} \right| \le \varepsilon \right] \ge 1 - 2e^{-2n\varepsilon^2}.  
$
\end{lemma}
Lemma~\ref{lem:DKW1} implies that, with probability at least $1 - 2e^{-2n\eps^2}$, for every $i\in [n]$,$\left| F(p_{(i)}) - \frac{n-i}{n} \right| \le \eps.$
Since $F^{-1}$ is assumed to be $M$-Lipschitz continuous,
\begin{equation}
    \left| p_{(i)} - F^{-1}\left(\frac{n-i}{n}\right) \right| \le M \eps.
    \label{BoundG}
\end{equation}
We are now ready to prove \Cref{thm:GeneralizationBoundL}. We show the lower bound here; the proof for the upper bound uses similar techniques and is relegated to \Cref{app:proofGeneralizationU}.
\begin{proof}[Proof of \Cref{thm:GeneralizationBoundL}]
We will show that both the lower bound $c_Hn - O(1) \le K^{\star}$ and the upper bound $K^{\star} \le c_Ln + O(1)$ each occur with probability at least $1 - 2e^{-2n\eps^2}$ which, by a union bound, proves the desired probability. As previously mentioned, we will only prove the lower bound here.
Fix arbitrary odd $k$ and $n$ with $k \le n$ where $k = 2\ell + 1$ for some non-negative integer $\ell$. We will give sufficient conditions as a function of $n$ and $k$ for which we can apply \Cref{lem:new-core}.

First, by \Cref{Recursion} with $j = \ell$,
$\Pr[\mathcal{E}_{k}^{\ell}]
	    = \frac{1}{k-\ell} \sum_{\substack{S \subseteq [k] \\ |S| = \ell+1}} w(S) \sum_{i \in S} \frac{1 - p_{(i)}}{p_{(i)}}$. 
Because the support of $\dist$ is upper-bounded by $H$, $p_{(i)} \le H$ for all $i$ with probability one. So,
$
	    \sum_{i \in S} \frac{1 - p_{(i)}}{p_{(i)}} \ge  (\ell+1) \frac{1-H}{H}.
$
Noting that $\ell + 1 = \frac{k+1}{2} = k - \ell$ and $\Pr[\mathcal{E}_{k}^{\ell+1}]
	    = \sum_{\substack{S \subseteq [k] \\ |S| = \ell+1}} w(S)$, after rearranging we have
$
	    \frac{\Pr[\mathcal{E}_{k}^{\ell + 1}]}{\Pr[\mathcal{E}_{k}^\ell]} \le \frac{H}{1-H}.
$
Further, we note that
$\frac{p_{(k+1)}p_{(k+2)}}{(1-p_{(k+1)})(1-p_{(k+2)})} \ge \frac{p_{(k+2)}^2}{(1-p_{(k+2)})^2}$.

Now, if we want to apply the first item of \Cref{lem:new-core} to show some $k$ is not optimal, it suffices to require that 
\begin{equation}
    \frac{p_{(k+2)}^2}{(1-p_{(k+2)})^2} > \frac{H}{1-H} \iff p_{(k+2)} > \frac{1}{1+\sqrt{\frac{1-H}{H}}}. 
    \label{eq:p-sufficient-conditionLower}
\end{equation}
Relying on \Cref{BoundG}, it holds that 
$ p_{(k+2)} \ge F^{-1}\left(\frac{n-k-2}{n}\right) - M\eps. $
If we require 
\begin{equation} \label{eq:k-sufficient-conditionLower}
F^{-1}\left(\frac{n-k-2}{n}\right) - M\eps > \frac{1}{1+\sqrt{\frac{1-H}{H}}},
\end{equation}
then \Cref{eq:p-sufficient-conditionLower} is satisfied and hence so will the condition of \Cref{lem:new-core}, which implies that such $k$ cannot be optimal. Solving \Cref{eq:k-sufficient-conditionLower} gives $\frac{k}{n} \le 1 - F\left(\frac{1}{1+\sqrt{\frac{1-H}{H}}} + M\eps\right) - \frac{2}{n}$, so
\begin{equation*}
    \frac{K^\star}{n} \ge 1 - F\left(\frac{1}{1+\sqrt{\frac{1-H}{H}}} + M\eps\right) - \frac{2}{n}.
\end{equation*}
Multiplying by $n$ yields the desired lower bound.
\end{proof}
This proves that for competencies drawn from an arbitrary distribution whose support is bounded away from 1, a constant fraction of the total population is needed to maximize the probability that the representatives make the correct decision on behalf of the entire population.

We illustrate \Cref{thm:GeneralizationBoundL} by distribution
$\mathcal D = \mathcal U(0.1, 0.9)$.  Letting $\eps=\sqrt{\frac{\log n}{2n}}$,
it can be checked that
$
0.186n \le K^{\star} \le 0.813n
$
with probability at least $1-\frac{4}{n}$ for all sufficiently large $n$.

\section{Can a Small Congress Outperform Direct Voting?}
\label{sec:outperform}
Our theoretical results from the previous section suggest that the optimal size of a congress should be linear in the size of the population. However, for many scenarios this may not be feasible and there are many other desiderata one must consider in choosing an ``optimal'' size. Hence, we now turn to \emph{comparing} how well different sizes of congresses perform in the epistemic model.

As a baseline, we will compare the accuracy of a congress to the accuracy of \emph{direct democracy} in which all $n$ members of society vote. This is well-motivated by classic results such as the \emph{Condorcet Jury Theorem} and extensions thereof, 
which show that the entire society will converge to the correct answer if and only if the competency distribution is \emph{biased} toward the correct answer, that is, $\E_{p\sim \dist}[p] > 1/2$.
We aim to find bounds on \emph{how} biased this distribution must be in order for congresses of different sizes to outperform the entire society. 

Now we state our problem formally.
We will be interested in how the cutoff of the bias of the competency distribution varies with $n$, hence, we will allow the distribution $\dist$ to depend on $n$ by having a distribution $\dist_n$ for each $n$.  We use $F_n$ and $f_n$ to denote the CDF and PDF of $\dist_n$ respectively.  
Let $\Gamma^{\bm p}_n(k)$ be the gain in probability of correctness by using a congress of size $k$ instead of the entire population, given competence levels $\bm p = (p_{(1)}, \ldots, p_{(n)})$:
\begin{equation*}
    \Gamma^{\bm p}_n(k) =\Pr\left[\sum_{i=1}^k X_{(i)} > \frac{k}{2}~\bigg|~ X_{(i)} \sim \mathrm{Bern}(p_{(i)})\right] - \Pr\left[\sum_{i=1}^n X_{(i)} > \frac{n}{2}~\bigg|~ X_{(i)} \sim \mathrm{Bern}(p_{(i)})\right].
\end{equation*}

Similar to the definition of $K^\star$, $\Gamma^{\bm p}_n(k)$ is a random variable whose randomness comes from the random draws  of $p_i \sim \dist_n$.
We aim at identifying, for certain values of $k$, for what kinds of distributions $\dist_n$ we have $\Gamma^{\bm p}_n(k) > 0$ with high probability as $n$ grows large.

\subsection{Dictatorship}

First, we consider an extreme case: when can a single voter outperform the entire society?  In particular, we identify conditions under which $\Gamma^{\bm p}_n(1) > 0$ or $\Gamma^{\bm p}_n(1) < 0$. 
We show that if the distributions $\dist_n$ put high enough probability mass on competence levels near $1$ and its mean $\E_{\dist_n}[p]$ is not much larger than $1/2$, then $\Gamma^{\bm p}_n(1) > 0$ with high probability as $n$ grows large, and $\Gamma^{\bm p}_n(1) < 0$ on the contrary. The probability mass conditions are satisfied by many natural classes of distributions; we give several examples (e.g., uniform and beta distributions) in \Cref{app:dists-satsifying-dictatorship}.

\begin{theorem} \label{thm:dictatorship}
Let $k=1$. 
\vspace{-0.4em}
\begin{itemize}
    \item Suppose $\E_{\dist_n}[p] \le \frac{1}{2} + a \sqrt{\frac{\log n}{n}}$ and $f_n(x) \ge \underline C (1-x)^{\underline \beta - 1}$ for $x\in [1-\underline \delta, 1]$ for some constants $a, \underline C, \underline \beta, \underline \delta > 0$.  If $a < \sqrt{ \E_{\dist_n}[p(1-p)] \cdot \min\{1, 2/{\underline{\beta}} \} }$, then, with probability at least $1 - n^{-\Omega(1)}$, $\Gamma^{\bm p}_n(1) > 0$.
\vspace{-0.3em}
    \item Suppose $\E_{\dist_n}[p] \ge \frac{1}{2} + a \sqrt{\frac{\log n}{n}}$ and $f_n(x) \le \overline C$ for $x\in [1-\overline \delta, 1]$ for some constants $a, \overline C, \overline \delta > 0$.  If $a > \frac{1}{\sqrt 2}$, then with probability at least $1 - n^{-\Omega(1)}$, $\Gamma^{\bm p}_n(1) < 0$.
\vspace{-0.1em}
\end{itemize}
\end{theorem}

We sketch a proof of the theorem; the full proof is in \Cref{app:proof-general-k}. When $\E_{\dist_n}[p] = \frac{1}{2} + O(\sqrt{\frac{\log n}{n}})$, by Hoeffding's inequality, the entire population makes a correct decision with probability $1 - O(n^{-c_1})$ for some constant $c_1$, while by our assumption on $\dist_n$ the top expert is correct with probability $p_{(1)} = 1 - O(n^{-c_2})$.  We identify conditions on $\dist_n$ for which $c_1<c_2$ or $c_1 > c_2$.

\subsection{Real-world and Polynomial-sized Congress}
\label{Simulations}

We turn our attention to more practical congress size. As discussed in the introduction, prior work has suggested that the size of congress should be near the cube-root of the population size. Exploring real-world data for 240 legislatures,\footnote{The data comes from Wikipedia: \url{https://en.wikipedia.org/wiki/List_of_legislatures_by_number_of_members}. We consider the number of representatives to be the total number of representatives in both chambers.}, we re-ran regression analysis of \citet{auriol2012optimal} on the log of the congress sizes of many countries compared to the log of the population size, which yields a slope of $0.36$ (with intercept $-0.65$ and coefficient of determination $R^2 = 0.85$)), suggesting $k = \Theta(n^{0.36})$. See results in \Cref{app:real-world}.

Next, we numerically investigate how congresses of this size perform compared to direct democracy with different levels of bias. We consider $k = n^{0.36}$ and $\dist_n = \mathcal{U}(L+\eps_n, 1 - L)$ such that $\E_{\dist_n}[p] = \frac{1 + \eps_n}{2}$. So the society is slightly biased toward the correct answer. We identify sequences $(\eps_n)_{n=1}^\infty$ such that a congress of size $k$ outperforms direct democracy for sufficiently large $n$.

The simulations were ran on a MacBook Pro as follows: for a given distribution, we sample $n$ competencies and votes associated with these competencies. We perform two majority votes --- with all the voters and with the top $k$ voters. Repeating this operation $1,000$ times, 
we estimate the probabilities that the majority of all voters (Direct Democracy) and $k$ voters (Representative Democracy) are correct. 
\Cref{fig:simT} displays the probabilities and $95\%$ confidence intervals for different population sizes, with $L=0.4$. Additional simulations can be found in \Cref{app:outperforming-figures}.

\begin{figure}[htb]
    \centering
    \includegraphics[width=1\textwidth]{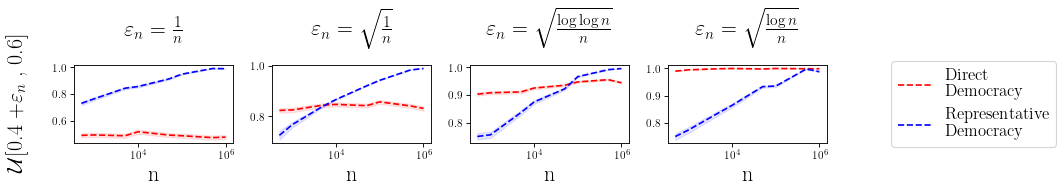}
    \caption{Estimates of $\Pr[\sum_{i=1}^k X_{(i)} > \frac{k}{2}~|~ \bm p]$ (Representative Democracy) and $\Pr[\sum_{i=1}^n X_{(i)} > \frac{n}{2}~|~ \bm p]$ (Direct Democracy) as a function of the population size for different values of $\eps_n$, with $k=n^{0.36}$ and $\dist_n = \mathcal{U}[0.4+\eps_n, 0.6].$ For large 
    $\eps_n$,
    the population size needs to reach a critical mass for the congress to outperform direct democracy.}
    \label{fig:simT}
\end{figure}

Let us now formalize and prove this result for general distributions. If the average competence level of the population, $\E_{\dist_n}[p]$, is larger than $\frac{1}{2}$ by a constant margin, then both the entire population and a congress of size $n^r$ will be correct with probabilities that are exponentially close to $1$.  Hence, again, to make things more interesting, we are concerned with the case where $\E_{\dist_n}[p] = \frac{1}{2} + \eps_n$ with $0 < \eps_n < o(1)$.  We identify conditions on $\eps_n$, $n$ and $\dist_n$ under which $\Gamma^{\bm p}_n(k) > 0$ or $\Gamma^{\bm p}_n(k) < 0$. The following result is proved in \Cref{app:proof-general-k}.
\begin{theorem}\label{thm:general-k-comparison}
Let $k = n^r$ for some constant $0<r<1$. 
\vspace{-0.2em}
\begin{itemize} 
    \item Suppose $\E_{\dist_n}[p] \le \frac{1}{2}  + a \sqrt{\frac{\log n}{n}}$, and $1 - F_n(\frac{1}{2} + \alpha \sqrt{\frac{\log k}{k}}) \ge \frac{k}{n} + \Omega(\sqrt{\frac{\log n}{n}})$ for some constants $a, \alpha > 0$.  If $a<\sqrt{\E_{\dist_n}[p(1-p)]}$ and $\alpha > \frac{a}{2\sqrt{r\cdot \E_{\dist_n}[p(1-p)]}}$, then, with probability at least $1-n^{-\Omega(1)}$, $\Gamma^{\bm p}_n(k) > 0$. 
\vspace{-0.2em}
    \item Suppose $\E_{\dist_n}[p] \ge \frac{1}{2} + a \sqrt{\frac{\log n}{n}}$ and $1-F_n(\frac{1}{2} + \alpha\sqrt{\frac{\log k} k }) \le \frac{1}{n^{1+\Omega(1)}}$ for some constants $a, \alpha>0$.  If $\alpha < \frac 1 2$ and $a > \sqrt{r} \alpha$, then, with probability at least $1-n^{-\Omega(1)}$, $\Gamma^{\bm p}_n(k) < 0$. 
\vspace{-0.3em}
\end{itemize}
\end{theorem}
Intuitively, in the first item above, the condition on the CDF, $1 - F_n(\frac{1}{2} + \alpha \sqrt{\frac{\log k}{k}}) \ge \frac{k}{n} + \Omega(\sqrt{\frac{\log n}{n}})$, and the condition on $\alpha$ imply that $\dist_n$ assigns large enough probability to high competence levels $p > \frac{1}{2} + \alpha \sqrt{\frac{\log k}{k}}$, so a congress of size $n^r$ will be composed of competent enough experts and hence will beat the entire population.  The conditions in the second item are in the opposite direction. 

We remark that the above conditions on the relation between $a$ and $\alpha$ are sharp: for distributions $\dist_n$ that are concentrated around $1/2$, we have $\E_{\dist_n}[p(1-p)] \approx 1/4$, so the first condition becomes $\alpha > \frac{a}{2\sqrt{r\cdot 1/4}} = \frac{a}{\sqrt r}$, or equivalently $a < \sqrt{r} \alpha$, while the second condition is the opposite: $a > \sqrt{r} \alpha$. 

Finally, we note that the conditions in \Cref{thm:general-k-comparison} on the distributions $\dist_n$ are satisfied by many natural classes of distributions, e.g., beta distributions and normal distributions truncated to $[0, 1]$. We identify more examples in \Cref{app:dists-satsifying-general-k}.
\section{Discussion}
\label{Discussion}
We have proved that under mild conditions, through the lens of an epistemic approach, current congresses are run with a sub-optimal size. However, despite this, it seems that these smaller congresses can still be cogent by at least beating majority under appropriate conditions. 

Current debates about the number of representatives in democracies tend to be about reducing their size, not increasing.\footnote{A 2020 Italian referendum approved reducing congress' size from 945 to 600~\citep{945}.} Indeed, even under the assumption that a larger congress would lead to a ``correct'' answer more often, this is clearly not the only desiderata to consider. Even under the strong assumption that the congress-members' votes reflect those of the top experts in society, congress-members are costly for the taxpayers. Beyond this, the legitimacy and representativeness of the institution are constantly under scrutiny. Designing political institutions relying solely on mathematical insights could yield unforeseen negative externalities (did Madison not warn against the \textit{confusion of the multitude}?). Cognitive, sociological and economical knowledge should be coupled with mathematical analyses to reach a reasonable trade-off, rather than optimizing a single factor. 

Further research could study the range of $k$ such that the probability that $k$ experts are right is close to or approximates the maximal probability. Incorporating a cost analysis, similar to \citet{magdon-ismail_mathematical_2018}, also seems particularly relevant to quantify the trade-off between the congress accuracy and its costs for the constituents.

Finally, this work supports, to some extent, propositions to constitute assemblies of citizens under fluid democracy\footnote{Fluid democracy relies on letting citizens nominating someone to represent themselves directly or to self-select to participate in the assembly with a weight equals to the number of votes she transitively gathered. }~\citep{miller, blum2016liquid, green, christoff, kahng, golz} that would vote on behalf of the entire population. Indeed, fluid democracy could yield very large citizen assemblies deemed desirable by our findings. Further research on the accuracy of such citizen assemblies could discuss the influence of the voters' weight in the weighted majority's performance. 
\newpage
\bibliographystyle{plainnat}
\bibliography{abb,bibfile}

\newpage
\appendix
\section*{Appendix}

\section{Missing Proofs}\label{app:missing-proof}

\subsection{Missing Portion of Proof of \Cref{thm:GeneralizationBoundL}}
\label{app:proofGeneralizationU}
Symmetric to the lower bound, we have that
$$
	\frac{\Pr[\mathcal{E}_{k}^{\ell + 1}]}{\Pr[\mathcal{E}_{k}^\ell]} \ge \frac{L}{1-L}.
	$$
Further, 	
$$\frac{p_{(k+1)}p_{(k+2)}}{(1-p_{(k+1)})(1-p_{(k+2)})} \le \frac{p_{(k+1)}^2}{(1-p_{(k+1)})^2}.
	$$
	
Hence, to prove a certain value $k + 2$ is not optimal using \Cref{lem:new-core}, it suffices that
\begin{equation*}
    \frac{p_{(k+1)}^2}{(1-p_{(k+1)})^2} < \frac{1-L}{L}, 
\end{equation*} which is equivalent to
\begin{equation}
\label{eq:p-sufficient-conditionUpper}
    p_{(k+1)} < \frac{1}{1+\sqrt{\frac{L}{1-L}}}
\end{equation}
Now, relying on \Cref{BoundG}, it holds that 
\[  p_{(k+1)} \le F^{-1}(\frac{n-k-1}{n}) + M\eps. \]
If we require 
\begin{equation}
\label{eq:k-sufficient-conditionUpper}
F^{-1}(\frac{n-k-1}{n}) + M\eps < \frac{1}{1+\sqrt{\frac{1-L}{L}}},
\end{equation}
then \Cref{eq:p-sufficient-conditionUpper} is satisfied and hence $p_{n}^{k+2} - p_n^k < 0$, which implies that such $k$ cannot be optimal.  Solving  \Cref{eq:k-sufficient-conditionUpper} gives $\frac{k}{n} > 1 - F\left(\frac{1}{1+\sqrt{\frac{1-L}{L}}} - M\eps\right) - \frac{1}{n}$. 

Hence, as long as $\frac{k}{n} > 1 - F\left(\frac{1}{1+\sqrt{\frac{1-L}{L}}} - M\eps\right) - \frac{1}{n}$, the condition of \Cref{lem:new-core} will be satisfied. Multiplying through by $n$ yields the desired upper bound. \qed

\subsection{Proof of \Cref{thm:dictatorship}}\label{app:proof-dictatorship-new}

For the proof, we will need the following lemmas, the first and third are well-known concentration inequalities and the second is a standard bound on the standard normal CDF which we prove here for completeness.

\begin{lemma}[Berry-Esseen Theorem]\label{lem:BE}
Let $X_1, \ldots, X_n$ be independent random variables with $\E[X_i] = 0$, $\E[X_i^2] = \sigma_i^2 > 0$, and $\E[|X_i|^3] = \rho_i < \infty$.  Let $F_{S_n}$ be the CDF of $S_n = \frac{\sum_{i=1}^n X_i}{\sqrt{\sum_{i=1}^n \sigma_i^2}}$ and $\Phi$ be the CDF of the standard normal distribution.  Then, there exists an absolute constant $C_1$ such that 
\[  | F_{S_n}(x) - \Phi(x) | \le \frac{C_1}{\sqrt{\sum_{i=1}^n \sigma_i^2}} \max_{1\le i\le n} \frac{\rho_i}{\sigma_i^2}, ~~~\forall x\in \reals \]
\end{lemma}

\begin{lemma}[Bounds on standard normal CDF]\label{lem:normal-CDF}
Let $\Phi(x) = \int_{-\infty}^x \frac{1}{\sqrt{2\pi}}e^{-t^2/2} \dd t$ be the CDF of the standard normal distribution.  Then we have for any $x>0$, 
\[ \frac{1}{\sqrt{2\pi}} \frac{x}{x^2 + 1} e^{-x^2/2} \le \Phi(-x) = 1 - \Phi(x) \le \frac{1}{\sqrt{2\pi}}\frac{1}{x} e^{-x^2/2}.  \]
\end{lemma}
\begin{proof}
The right inequality is because 
\begin{align*}
 1 - \Phi(x) & = \int_{x}^{\infty} \frac{1}{\sqrt{2\pi}} e^{-\frac{t^2}{2}} \dd t \\
    & \le \int_{x}^{\infty} \frac{1}{\sqrt{2\pi}} \frac{t}{x} e^{-\frac{t^2}{2}} \dd t = \frac{1}{\sqrt{2\pi}} \frac{1}{x} \left(- e^{-\frac{t^2}{2}}\right)\bigg|_{t=x}^{\infty} \\
    & = \frac{1}{\sqrt{2\pi}} \frac{1}{x} e^{-\frac{x^2}{2}}. 
\end{align*}

The left inequality is because 
\begin{align*}
    1-\Phi(x) & = \int_{x}^{\infty} \frac{1}{\sqrt{2\pi}} e^{-\frac{t^2}{2}} \dd t \\
    & \ge \int_{x}^{\infty} \frac{1}{\sqrt{2\pi}} \frac{(t^2+1)^2 - 2}{(t^2+1)^2} e^{-\frac{t^2}{2}} \dd t = \frac{1}{\sqrt{2\pi}}\left(-\frac{t}{t^2+1} e^{-\frac{t^2}{2}}\right)\bigg|_{t=x}^{\infty} \\
    & = \frac{1}{\sqrt{2\pi}} \frac{x}{x^2 + 1} e^{-\frac{x^2}{2}}.\qedhere
\end{align*}
\end{proof}

\begin{lemma}[Hoeffding's Inequality]\label{lem:hoeffding}
Let $X_1, \ldots, X_n$ be independent random variables bounded by $0\le X_i \le 1$.  Then
\[  \Pr\left[ \sum_{i=1}^n X_i  \ge \E[\sum_{i=1}^n X_i] + t\right] \le \exp(-\frac{2t^2}{n}),  \]
for any $t > 0$.  The other direction also holds: 
\[  \Pr\left[ \sum_{i=1}^n X_i \le \E[\sum_{i=1}^n X_i] -t \right] \le \exp(-\frac{2t^2}{n}).  \]
\end{lemma}

Now we prove \Cref{thm:dictatorship}.

\begin{proof}[Proof of \Cref{thm:dictatorship}]
To simplify notations we write $\Pr\left[\sum_{i=1}^k X_{(i)} > \frac{k}{2} \mid X_{(i)} \sim \mathrm{Bern}(p_{(i)})\right]$ as $\Pr\left[\sum_{i=1}^k X_{(i)} > \frac{k}{2} \mid \bm p\right]$.
Recalling the definition of $\Gamma^{\bm p}_n(k)$, since $\Pr\left[\sum_{i=1}^k X_{(i)} > \frac{k}{2} \mid \bm p\right] = 1 - \Pr\left[\sum_{i=1}^k X_{(i)} \le \frac{k}{2} \mid \bm p\right]$, $\Gamma^{\bm p}_n(k)$ can be equivalently written as 
\[ \Gamma^{\bm p}_n(k) = \Pr\left[\sum_{i=1}^n X_{(i)} \le \frac{n}{2} ~\bigg|~ \bm p\right] - \Pr\left[\sum_{i=1}^k X_{(i)} \le \frac{k}{2} ~\bigg|~ \bm p\right]. \]
To show either $\Gamma^{\bm p}_n(k) > 0$ or $\Gamma^{\bm p}_n(k) < 0$, we will compare $\Pr\left[\sum_{i=1}^n X_{(i)} \le \frac{n}{2} ~|~ \bm p\right]$ with $\Pr\left[\sum_{i=1}^k X_{(i)} \le \frac{k}{2} ~|~ \bm p\right]$.  To do this, we prove the following lemmas: 

\begin{lemma}\label{lem:Pr-n-lower-bound}
Suppose $\E_{p\sim \dist_n}[p] \le \frac{1}{2} + \eps_n$ where $\eps_n = a \sqrt{\frac{\log n}{n}}$ for some constant $a>0$.  Let $\eps = b \sqrt{\frac{\log n}{n}}$ for some constant $b>0$.  Suppose $\E_{p\sim \dist_n}[p(1-p)] > \eps$.  
Let $c = \frac{a + b}{\sqrt{\E_{p\sim \dist_n}[p(1-p)] - \eps}}$.  Then we have: with probability at least $1 - 2n^{-2b^2}$ (over the random draw of $\bm p \sim \dist_n$), 
\[ \Pr\left[\sum_{i=1}^n X_{(i)} \le \frac{n}{2} ~\bigg|~ \bm p\right] \ge \frac{1}{\sqrt{2\pi}} \cdot \frac{c \sqrt{\log n}}{(c^2\log n + 1)} \cdot \frac{1}{n^{c^2/2}} - \frac{C_1}{\sqrt{\E_{p\sim \dist_n}[p(1-p)] - \eps}}\frac{1}{\sqrt{n}}, \]
where $C_1$ is the constant in Berry-Esseen theorem (\Cref{lem:BE}). 
\end{lemma}
\begin{proof}
Given $\bm p = (p_{(1)}, \ldots, p_{(n)})$, each $X_{(i)}$ independently follows $\text{Bern}(p_{(i)})$.  We use Berry-Esseen theorem (\Cref{lem:BE}) for $Y_i = X_{(i)} - p_{(i)}$, $i=1, \ldots, n$.  Noticing that $\E[Y_i] = 0$, $\sigma_i^2 = \E[Y_i^2] = p_{(i)}(1-p_{(i)})$, and $\rho_i = \E[|Y_i|^3] = p_{(i)}(1-p_{(i)})[(1 - p_{(i)})^2 + p_{(i)}^2 \le \sigma_i^2$, the theorem implies 
\[ \left| \Pr\left[ \frac{\sum_{i=1}^n Y_i}{\sum_{i=1}^n \sigma_i^2} \le x \right] - \Phi(x) \right| \le \frac{C_1}{\sqrt{\sum_{i=1}^n \sigma_i^2}} \max_{1\le i \le n}\frac{\rho_i}{\sigma_i^2} \le \frac{C_1}{\sqrt{\sum_{i=1}^n \sigma_i^2}} \stackrel{\text{def}}{=} \Delta_1  \]
for any $x\in \reals$, where $\Phi(x)$
is CDF of the standard normal distribution.  
Therefore, 
\begin{align}
    \Pr\left[\sum_{i=1}^n X_{(i)} \le \frac{n}{2} ~\bigg|~ \bm p\right] & = \Pr\left[ \sum_{i=1}^n X_{(i)} - \sum_{i=1}^n p_{(i)} \le \frac{n}{2} - \sum_{i=1}^n p_{(i)} ~\bigg|~ \bm p\right] \nonumber \\
    & = \Pr\left[ \frac{\sum_{i=1}^n Y_i}{\sqrt{\sum_{i=1}^n \sigma_i^2}} \le \frac{\frac{n}{2} - \sum_{i=1}^n p_{(i)}}{\sqrt{\sum_{i=1}^n \sigma_i^2}} \right] \nonumber \\
    & \ge \Phi\left(\frac{\frac{n}{2} - \sum_{i=1}^n p_{(i)}}{\sqrt{\sum_{i=1}^n \sigma_i^2}}\right) - \Delta_1  \label{eq:probability-p-sigma-phi}
\end{align}
We note that $\sum_{i=1}^n p_{(i)} = \sum_{i=1}^n p_i$ is the sum of $n$ i.i.d.~draws from distribution $\dist_n$, with mean $\E[\sum_{i=1}^n p_i] = n \E_{p\sim \dist_n}[p]$.  By Hoeffding's inequality (\Cref{lem:hoeffding}), letting $t = n\eps$, we have
\begin{equation}\label{eq:p-mean}
    \sum_{i=1}^n p_i  \le n \E_{p\sim \dist_n}[p] + n\eps 
\end{equation}
with probability at least $1 - \exp(-\frac{2(n\eps)^2}{n}) = 1 - n^{-2b^2}$.  Also, $\sum_{i=1}^n \sigma_i^2 = \sum_{i=1}^n p_i(1-p_i)$ is the sum of $n$ i.i.d.~draws from a distribution, with mean $n \E_{p\sim \dist_n}[p(1-p)]$, so 
\begin{equation}\label{eq:p-sigma}
    \sum_{i=1}^n \sigma_i^2  \ge n \E_{p\sim \dist_n}[p(1-p)] - n\eps 
\end{equation}
also with probability at least $1 - \exp(-\frac{2(n\eps)^2}{n}) = 1 - n^{2b^2}$.  By a union bound, we have with probability at least $1 - 2n^{-2b^2}$, both \Cref{eq:p-mean} and \Cref{eq:p-sigma} hold, which imply
\begin{align*}
    \Pr\left[\sum_{i=1}^n X_{(i)} \le \frac{n}{2} ~\bigg|~ \bm p\right] & \ge \Phi\left(\frac{\frac{n}{2} - \sum_{i=1}^n p_{(i)}}{\sqrt{\sum_{i=1}^n \sigma_i^2}}\right) - \Delta_1 \\
    & \ge \Phi\left(\frac{\frac{n}{2} - n \E_{p\sim \dist_n}[p] - n\eps}{\sqrt{\sum_{i=1}^n \sigma_i^2}}\right) - \Delta_1 \\
    & \ge \Phi\left(\frac{- n\eps_n - n\eps}{\sqrt{\sum_{i=1}^n \sigma_i^2}}\right) - \Delta_1 \\
    & \ge \Phi\left(\frac{- n\eps_n - n\eps}{\sqrt{n \E_{p\sim \dist_n}[p(1-p)] - n\eps}}\right) - \Delta_1 \\
    & = \Phi\left(- \sqrt n \frac{\eps_n + \eps}{\sqrt{\E_{p\sim \dist_n}[p(1-p)] - \eps}}\right) - \Delta_1 \\
    & =  \Phi\left(- \sqrt n \frac{(a+b)\sqrt{\frac{\log n}{n}}}{\sqrt{\E_{p\sim \dist_n}[p(1-p)] - \eps}}\right) - \Delta_1 \\
    & =  \Phi\left(- c \sqrt{\log n}\right) - \Delta_1.
\end{align*}
Using \Cref{lem:normal-CDF} with $x = c \sqrt{\log n}$, we get
\[\Pr\left[\sum_{i=1}^n X_{(i)} \le \frac{n}{2} ~\bigg|~ \bm p\right] \ge \Phi\left(- c \sqrt{\log n}\right) - \Delta_1 \ge \frac{1}{\sqrt{2\pi}} \frac{c\sqrt{\log n}}{c^2\log n + 1} \frac{1}{n^{c^2/2}} - \Delta_1, \]
concluding the proof. 
\end{proof}

\begin{lemma}\label{lem:Pr-n-upper-bound}
Suppose $\E_{p\sim \dist_n}[p] \ge \frac{1}{2} + \eps_n$ where $\eps_n = a \sqrt{\frac{\log n}{n}}$ for some constant $a>0$.
Let $b$ be a constant with $0 < b < a$. 
Then we have: with probability at least $1 - n^{-2b^2}$ (over the random draw of $\bm p \sim \dist_n$), 
\[ \Pr\left[\sum_{i=1}^n X_{(i)} \le \frac{n}{2} ~\bigg|~ \bm p\right] \le \frac{1}{n^{2(a-b)^2}}.  \]
\end{lemma}
\begin{proof}
We note that $\sum_{i=1}^n p_{(i)} = \sum_{i=1}^n p_i$ is the sum of $n$ i.i.d.~draws from distribution $\dist_n$, with mean $\E[\sum_{i=1}^n p_i] = n \E_{p\sim \dist_n}[p]$.  Let $\eps = b\sqrt{\frac{\log n}{n}} < \eps_n$.  By Hoeffding's inequality (\Cref{lem:hoeffding}), with probability at least $1 - \exp(-\frac{2(n\eps)^2}{n}) = 1 - n^{-2b^2}$, it holds that 
\begin{equation*}
    \sum_{i=1}^n p_i \ge  n \E_{p\sim \dist_n}[p] -  n\eps \ge \frac{n}{2} + n\eps_n - n\eps > \frac{n}{2}. 
\end{equation*}
Assuming $\sum_{i=1}^n p_i \ge  n \E_{p\sim \dist_n}[p] -  n\eps$ holds,
we consider the conditional probability  $\Pr[\sum_{i=1}^n X_{(i)} \le \frac{n}{2} ~|~ \bm p]$.  Given $\bm p$, $X_{(i)}$'s are independent Bernoulli random variables with means $\E[X_{(i)}] = p_{(i)}$. Hence, by Hoeffding's inequality (\Cref{lem:hoeffding}), 
\begin{align*}
\Pr\left[\sum_{i=1}^n X_{(i)} \le \frac{n}{2} ~\bigg|~ \bm p\right]
& \le \exp\left(-\frac{2(\sum_{i=1}^n p_{(i)} - \frac{n}{2})^2}{n}\right) \\
& \le \exp\left(-\frac{2(n\eps_n - n\eps)^2}{n}\right) = \exp\left(-2n(\eps_n-\eps)^2\right) = \frac{1}{n^{2(a-b)^2}}. \qedhere
\end{align*}
\end{proof}

\begin{lemma}\label{lem:Pr-1-upper-bound}
Suppose the PDF of $\dist_n$ satisfies $f_n(x) \ge \underline C (1-x)^{\underline \beta - 1}$ for $x\in [1-\underline \delta, 1]$ for some constants $\underline C, \underline \beta, \underline \delta > 0$.  Then, for sufficiently large $n$, with probability at least $1 - n^{-d}$ over the random draw of $\bm p \sim \dist_n$, 
\[ \Pr[X_{(1)} = 0 \mid \bm p] \le \left( \frac{\underline \beta d \log n}{\underline C n} \right)^{1/{\underline \beta}}. \]
\end{lemma}
\begin{proof}
We note that $\Pr[X_{(1)} = 0 \mid \bm p] = 1 - p_{(1)}$, so for any $x \in [0, 1]$, 
\begin{align*}
    \Pr[X_{(1)} = 0 \mid \bm p] \le x] = \Pr[1 - p_{(1)} \le x]  = \Pr[ p_{(1)} \ge 1 - x]
    & = 1 - \Pr[ p_{(1)} < 1 - x] \\
    & = 1 - \Pr[ \max_{1\le i\le n} p_i < 1 - x] \\
    & = 1 - F_n(1-x)^n. 
\end{align*}
We let $x$ be such that $F_n(1 - x) = 1 - \frac{d\log n}{n}$, i.e., $x = 1 - F^{-1}_n(1 - \frac{d\log n}{n})$, then $F_n(1-x)^n = (1 - \frac{d\log n}{n})^n \le e^{-d\log n} = n^{-d}$. So, with probability at least $1 - F_n(1-x)^n \ge 1 - n^{-d}$, we have
\[ \Pr[X_{(1)} = 0 \mid \bm p] \le x = 1 - F^{-1}_n\left( 1 - \frac{d\log n}{n}\right). \]
We then show that $1 - F^{-1}_n\left( 1 - \frac{d\log n}{n}\right) \le \left( \frac{\underline \beta d \log n}{\underline C n} \right)^{1/{\underline \beta}}$.  Define $G(t) = 1 - F_n(1-t)$ for $t\in[0, 1]$.  This implies
\[1 - F^{-1}_n(1-y) = G^{-1}(y)\]
for any $y \in [0, 1]$. 
We note that for $t$ sufficiently close to $1$, $f_n(x) \ge \underline C(1-x)^{\underline \beta-1}$ for any $x\in[1-t, 1]$, implying
\begin{align*}
    G(t) = 1 - F_n(1-t) = \int_{1-t}^1 f_n(x) \dd x \ge  \int_{1-t}^1 \underline C(1-x)^{\underline \beta-1} \dd x = \int_{0}^t \underline Cu^{\underline \beta-1} \dd u = \frac{\underline C}{\underline \beta} t^{\underline \beta}. 
\end{align*}
Let $\underline{G}(t) = \frac{\underline C}{\underline \beta} t^{\underline \beta}$.  We have $G(t) \ge \underline{G}(t)$ and $\underline{G}^{-1}(y) = (\frac{\underline \beta}{\underline C} y)^{1/{\underline \beta}}$.  Since $G(t) \ge \underline{G}(t)$ and $\underline{G}^{-1}(y)$ is increasing in $y$, we have
\[ G(t) \ge \underline{G}(t)  \implies \underline{G}^{-1}(G(t)) \ge t \implies \underline{G}^{-1}(y) \ge G^{-1}(y).   \]
Therefore,
\[ 1 - F^{-1}_n(1-y) = G^{-1}(y) \le \underline{G}^{-1}(y) = (\frac{\underline \beta}{\underline C} y)^{1/{\underline \beta}}. \]
Letting $y = \frac{d\log n}{n}$, we conclude that 
\[ \Pr[X_{(1)} = 0 \mid \bm p] \le 1 - F^{-1}\left( 1 - \frac{d\log n}{n}\right) \le \left(\frac{\underline \beta d\log n}{\underline Cn} \right)^{1/{\underline \beta}}. \qedhere \] 
\end{proof}

\begin{lemma}\label{lem:Pr-1-lower-bound}
Suppose the PDF of $\dist_n$ satisfies $f_n(x) \le \overline C$ for $x\in [1-\overline \delta, 1]$ for some constants $\overline C, \overline \delta > 0$.  Then, for sufficiently large $n$, with probability at least $1 - n^{-d}$ over the random draw of $\bm p \sim \dist_n$, 
\[ \Pr[X_{(1)} = 0 \mid \bm p] \ge \frac{1}{\overline C n^{d+1}}. \]
\end{lemma}
\begin{proof}
We note that $\Pr[X_{(1)} = 0 \mid \bm p] = 1 - p_{(1)}$, so for any $x \in [0, 1]$, 
\begin{align*}
    \Pr[\Pr[X_{(1)} = 0 \mid \bm p] \ge x] = \Pr[1 - p_{(1)} \ge x]  = \Pr[ p_{(1)} \le 1 - x]
    & = \Pr[ \max_{1\le i\le n} p_i < 1 - x] \\
    & = F_n(1-x)^n. 
\end{align*}
We let $x = \frac{1}{\overline C n^{d+1}}$.  Then for sufficiently large $n$, $x \ge 1-\overline \delta$, and hence $f_n(t) \le \overline C$ for $t\in[1-x, 1]$, which implies 
\[ 1 - F_n(1-x) = \int_{1-x}^1 f_n(t)\dd t \le \int_{1-x}^1 \overline C\dd t = x\overline C = \frac{1}{n^{d+1}}, \]
or equivalently
\[ F_n(1-x) \ge 1 - \frac{1}{n^{d+1}}. \]
Using inequality $(1-\frac{x}{n})^n \ge 1-x$ (for $n\ge 1, 0\le x\le n$), we get
\[ F_n(1-x)^n \ge \left(1 - \frac{1}{n^{d+1}}\right)^n \ge 1 - \frac{1}{n^d}. \]
Therefore, with probability at least $1 - \frac{1}{n^d}$, $\Pr[X_{(1)} = 0\mid \bm p] \ge x = \frac{1}{\overline C n^{d+1}}$ holds. 
\end{proof}

To prove $\Gamma^{\bm p}_n(1) > 0$, we use \Cref{lem:Pr-n-lower-bound} and \Cref{lem:Pr-1-upper-bound} to get
\begin{align*}
\Gamma^{\bm p}_n(1) & = \Pr\left[\sum_{i=1}^n X_{(i)} \le \frac{n}{2} ~\bigg|~ \bm p\right] - \Pr\left[X_{(1)} = 0 ~\bigg|~ \bm p\right] \\
& \ge \frac{1}{\sqrt{2\pi}} \frac{c \sqrt{\log n}}{(c^2\log n + 1)} \frac{1}{n^{c^2/2}} - \frac{C_1}{\sqrt{\E_{p\sim \dist_n}[p(1-p)] - \eps}}\frac{1}{\sqrt{n}} - \left( \frac{\underline \beta d \log n}{\underline C n} \right)^{1/{\underline \beta}}
\end{align*}
with probability at least $1 - 2n^{-2b^2} - n^{-d}$, where $c = \frac{a + b}{\sqrt{\E_{p\sim \dist_n}[p(1-p)] - \eps}}$, $\E_{p\sim \dist_n}[p] \le \frac{1}{2} + \eps_n$ with $\eps_n = a \sqrt{\frac{\log n}{n}}$ for some $a>0$, and $\eps = b \sqrt{\frac{\log n}{n}}$ for some $b>0$, and $\underline C$ and $\underline \beta$ are constants.  If $c^2/2$ is a constant such that 
\[ c^2/2 < \min\left\{ 1/2, 1/{\underline{\beta}} \right\}, \]
then $\Gamma^{\bm p}_n(1) =  O(\frac{1}{n^{c^2/2}}) > 0$ for sufficiently large $n$.  Requiring $c^2/2 < \min\left\{ 1/2, 1/{\underline{\beta}} \right\}$ is equivalent to requiring
\[ a + b < \sqrt{ ( \E_{p\sim \dist_n}[p(1-p)] - \eps) \cdot \min\{1, 2/{\underline{\beta}} \} }, \]
which can be satisfied when $a$ and $b$ are constants such that $a < \sqrt{ \E_{p\sim \dist_n}[p(1-p)] \cdot \min\{1, 2/{\underline{\beta}} \} }$, $0 < b < \sqrt{ \E_{p\sim \dist_n}[p(1-p)] \cdot \min\{1, 2/{\underline{\beta}} \} } -a$, and $n$ is sufficiently large (so $\eps=b\sqrt{\frac{\log n} n}$ is sufficiently small). 

To prove $\Gamma^{\bm p}_n(1) < 0$, we use \Cref{lem:Pr-n-upper-bound} and \Cref{lem:Pr-1-lower-bound} to get
\begin{align*}
\Gamma^{\bm p}_n(1) & = \Pr\left[\sum_{i=1}^n X_{(i)} \le \frac{n}{2} ~\bigg|~ \bm p\right] - \Pr\left[X_{(1)} = 0 ~\bigg|~ \bm p\right] \\
& \le \frac{1}{n^{2(a-b)^2}} - \frac{1}{\overline C n^{d+1}}
\end{align*}
with probability at least $1 - n^{-2b^2} - n^{-d}$, where $\E_{p\sim \dist_n}[p] \ge \frac{1}{2} + \eps_n$ with $\eps_n = a \sqrt{\frac{\log n}{n}}$ for some constant $a>0$, with any $b < a$, and $\overline C$ is a constant.  When
\[ 2(a-b)^2 > d+1, \]
we have $\Gamma^{\bm p}_n(1) = - O(\frac{1}{n^{d+1}}) < 0$ for sufficiently large $n$.  The inequality $2(a-b)^2 > d+1$ is satisfied when $a > \frac{1}{\sqrt 2}$ and $b, d$ are sufficiently close to $0$.  
\end{proof}

\subsection{Proof of \Cref{thm:general-k-comparison}}
\label{app:proof-general-k}
Similar to the proof of \Cref{thm:dictatorship} (in \Cref{app:proof-dictatorship-new}), we write $\Gamma^{\bm p}_n(k)$ as 
\[ \Gamma^{\bm p}_n(k) = \Pr\left[\sum_{i=1}^n X_{(i)} \le \frac{n}{2} ~\bigg|~ \bm p\right] - \Pr\left[\sum_{i=1}^k X_{(i)} \le \frac{k}{2} ~\bigg|~ \bm p\right]. \]
To show either $\Gamma^{\bm p}_n(k) > 0$ or $\Gamma^{\bm p}_n(k) < 0$, we will compare $\Pr\left[\sum_{i=1}^n X_{(i)} \le \frac{n}{2} ~|~ \bm p\right]$ with $\Pr\left[\sum_{i=1}^k X_{(i)} \le \frac{k}{2} ~|~ \bm p\right]$. 

\begin{lemma}\label{lem:Pr-k-upper-bound}
Suppose $1 - F_n(\frac{1}{2} + \alpha \sqrt{\frac{\log k}{k}}) \ge \frac{k}{n} + \eps$ where $\eps = b\sqrt{\frac{\log n}{n}}$ for some constants $\alpha, b > 0$.  Then, with probability at least $1 - 2n^{-2b^2}$ (over the random draw of $\bm p\sim \dist_n$),
\[ \Pr\left[\sum_{i=1}^k X_{(i)} \le \frac{k}{2} ~\bigg|~ \bm p\right] \le \frac{1}{k^{2\alpha^2}}. \]
\end{lemma}
\begin{proof}
By DKW inequality (Lemma 2.5), with probability at least $1 - 2e^{-2n\eps^2} = 1 - 2n^{-2b^2}$ over the random draw of $\bm p\sim \dist_n$, it holds that $| F_n(p_{(i)}) - \frac{n-i}{n} | \le \eps$ for every $i\in [n]$.  In particular, for $i=1, \ldots, k$, we have
\[ F_n(p_{(i)}) \ge  \frac{n-i}{n} - \eps \ge \frac{n-k}{n} -\eps = 1 - \frac{k}{n} - \eps, \]
This implies 
\[ 1 - F_n(p_{(i)}) \le  \frac{k}{n} + \eps \le 1 - F_n(\frac{1}{2} + \alpha \sqrt{\frac{\log k}{k}})\]
and hence 
\[ p_{(i)} \ge \frac{1}{2} + \alpha \sqrt{\frac{\log k}{k}}. \]
Assuming the above inequalities hold, we consider the conditional probability $\Pr\left[\sum_{i=1}^k X_{(i)} \le \frac{k}{2} ~|~ \bm p\right]$.  Given $\bm p$, the $X_{(i)}$'s are independent draws from $\text{Bern}(p_{(i)})$ distributions, with means $\E[X_{(i)}] = p_{(i)}$, hence, by Hoeffding's inequality (\Cref{lem:hoeffding}), 
\begin{align*}
    \Pr\left[\sum_{i=1}^k X_{(i)} \le \frac{k}{2} ~\bigg|~ \bm p\right] \le \exp\left(-\frac{2(\sum_{i=1}^k p_{(i)} - \frac{k}{2})^2}{k} \right).
\end{align*}
Plugging in $p_{(i)} \ge \frac{1}{2} + \alpha \sqrt{\frac{\log k}{k}}$, we get
\begin{align*}
    \Pr\left[\sum_{i=1}^k X_{(i)} \le \frac{k}{2} ~\bigg|~ \bm p\right] \le \exp\left(-\frac{2( \frac{k}{2} + \alpha \sqrt{k \log k} - \frac{k}{2})^2}{k} \right) = \frac{1}{k^{2\alpha^2}}. 
\end{align*}
\qedhere
\end{proof}

\begin{proof}[Proof of the first item of \Cref{thm:general-k-comparison}]
By Lemma~\ref{lem:Pr-n-lower-bound} and Lemma~\ref{lem:Pr-k-upper-bound}, we have
\begin{align*}
\Gamma^{\bm p}_n(k) & = \Pr\left[\sum_{i=1}^n X_{(i)} \le \frac{n}{2} ~\bigg|~ \bm p\right] - \Pr\left[\sum_{i=1}^k X_{(i)} \le \frac{k}{2} ~\bigg|~ \bm p\right] \\
& \ge \frac{1}{\sqrt{2\pi}} \frac{c \sqrt{\log n}}{(c^2\log n + 1)} \frac{1}{n^{c^2/2}} - \frac{C_1}{\sqrt{\E_{p\sim \dist_n}[p(1-p)] - \eps}}\frac{1}{\sqrt{n}} - \frac{1}{k^{2\alpha^2}}
\end{align*}
with probability at least $1 - 4n^{-2b^2}$, where $c = \frac{a + b}{\sqrt{\E_{p\sim \dist_n}[p(1-p)] - \eps}}$, $\E_{p\sim \dist_n}[p] \le \frac{1}{2} + \eps_n$ with $\eps_n = a \sqrt{\frac{\log n}{n}}$ for some $a>0$, and $\eps = b \sqrt{\frac{\log n}{n}}$ for some $b>0$, and $\alpha$ is a constant.  Since $k=n^r$,
\begin{align*}
\Gamma^{\bm p}_n(k) & \ge \frac{1}{\sqrt{2\pi}} \frac{c \sqrt{\log n}}{(c^2\log n + 1)} \frac{1}{n^{c^2/2}} - \frac{C_1}{\sqrt{\E_{p\sim \dist_n}[p(1-p)] - \eps}}\frac{1}{\sqrt{n}} - \frac{1}{n^{2r\alpha^2}}
\end{align*}
When $c^2/2 < 1/2$ and $c^2/2 < 2r\alpha^2$, we have $\Gamma^{\bm p}_n(k) = O(\frac{1}{n^{c^2/2}}) > 0$ for sufficiently large $n$.  The latter requirement $c^2/2 < 2r\alpha^2$ is satisfied when $\alpha > \frac{c}{2\sqrt r}$. The former requirement $c^2/2 < 1/2$ is equivalent to $a + b < \sqrt{ \E_{p\sim \dist_n}[p(1-p)] -\eps\} }$, which is satisfied when constants $a < \sqrt{ \E_{p\sim \dist_n}[p(1-p)]\} }$, $0 < b < \sqrt{ \E_{p\sim \dist_n}[p(1-p)]\} } - a$, and $n$ is sufficiently large. 
\end{proof}

\begin{lemma}\label{lem:Pr-k-lowerbound}
Suppose $1-F_n(\frac{1}{2}+\alpha\sqrt{\frac{\log k} k }) \le \frac{1}{n^{1+\Omega(1)}}$ for some constant $\alpha>0$, and suppose $\E_{p\sim \dist_n}[p] \ge \frac{1}{2} + \eps_n$ with $\eps_n = a\sqrt{\frac{\log n} n}$ for some constant $a > 0$.  Then, with probability at least $1-\frac{1}{n^{\Omega(1)}}$ (over the random draw of $\bm p\sim \dist_n$), 
\[ \Pr\left[\sum_{i=1}^k X_{(i)} \le \frac{k}{2} ~\bigg|~ \bm p\right] \ge \frac{1}{\sqrt{2\pi}}\cdot  \frac{1-o(1)}{2\alpha \sqrt{\log k}} \cdot \frac{1}{k^{\frac{2\alpha^2}{1-o(1)}}} - \frac{2C_1}{(1-o(1))\sqrt k},   \]
where $C_1$ is the constant in Berry-Esseen theorem (\Cref{lem:BE}). 
\end{lemma}

\begin{proof}
Given $p_{(1)}, \ldots, p_{(k)}$, each $X_{(i)}$ independently follows $\text{Bern}(p_{(i)})$.  We use Berry-Esseen theorem (\Cref{lem:BE}) for $Y_i = X_{(i)} - p_{(i)}$, $i=1, \ldots, k$.  Noticing that $\E[Y_i] = p_{(i)}$, $\sigma_i^2 = \E[Y_i^2] = p_{(i)}(1-p_{(i)})$, and $\rho_i = \E[|Y_i|^3] = p_{(i)}(1-p_{(i)})[(1 - p_{(i)})^2 + p_{(i)}^2] \le \sigma_i^2$, the theorem implies 
\[ \left| \Pr\left[ \frac{\sum_{i=1}^k Y_i}{\sum_{i=1}^k \sigma_i^2} \le x \right] - \Phi(x) \right| \le \frac{C_1}{\sqrt{\sum_{i=1}^k \sigma_i^2}} \max_{1\le i \le k}\frac{\rho_i}{\sigma_i^2} \le \frac{C_1}{\sqrt{\sum_{i=1}^k \sigma_i^2}}  \]
for any $x\in \reals$, where $\Phi(x)$ is CDF of the standard normal distribution.  
Therefore, 
\begin{align*}
    \Pr\left[\sum_{i=1}^k X_{(i)} \le \frac{k}{2} ~\bigg|~ \bm p\right] & = \Pr\left[ \sum_{i=1}^k X_{(i)} - \sum_{i=1}^k p_{(i)} \le \frac{k}{2} - \sum_{i=1}^k p_{(i)} ~\bigg|~ \bm p\right] \nonumber \\
    & = \Pr\left[ \frac{\sum_{i=1}^k Y_i}{\sqrt{\sum_{i=1}^k \sigma_i^2}} \le \frac{\frac{k}{2} - \sum_{i=1}^k p_{(i)}}{\sqrt{\sum_{i=1}^k \sigma_i^2}} \right] \nonumber \\
    & \ge \Phi\left(\frac{\frac{k}{2} - \sum_{i=1}^k p_{(i)}}{\sqrt{\sum_{i=1}^k \sigma_i^2}}\right) - \frac{C_1}{\sqrt{\sum_{i=1}^k \sigma_i^2}}.   \label{eq:probability-p-sigma-phi}
\end{align*}

We consider $\sum_{i=1}^k p_{(i)}$. 
By the assumption that $1 - F_n(\frac{1}{2} + \alpha\sqrt{\frac{\log k}k}) = \Pr_{p_i\sim \dist_n}[p_i > \frac{1}{2} + \alpha\sqrt{\frac{\log k}k}] = \frac{1}{n^{1+\Omega(1)}}$, using a union bound we have with probability at least $1 - n \frac{1}{n^{1+\Omega(1)}} = 1- \frac{1}{n^{\Omega(1)}}$, all $p_i$'s (for $i=1, \ldots, n$) satisfy $p_i \le \frac{1}{2} + \alpha\sqrt{\frac{\log k} k}$.  Hence, 
\[ \sum_{i=1}^k p_{(i)} \le k(\frac{1}{2} + \alpha\sqrt{\frac{\log k} k}) = \frac{k}{2} + \alpha \sqrt{k\log k},  \]
which implies 
\begin{align}
    \Pr\left[\sum_{i=1}^k X_{(i)} \le \frac{k}{2} ~\bigg|~ \bm p\right] & \ge \Phi\left(\frac{\frac{k}{2} - (\frac{k}{2} + \alpha \sqrt{k\log k})}{\sqrt{\sum_{i=1}^k \sigma_i^2}}\right) - \frac{C_1}{\sqrt{\sum_{i=1}^k \sigma_i^2}} \nonumber \\
    & = \Phi\left(\frac{-\alpha \sqrt{k\log k}}{\sqrt{\sum_{i=1}^k \sigma_i^2}}\right) - \frac{C_1}{\sqrt{\sum_{i=1}^k \sigma_i^2}} \label{eq:Pr-sigma}
\end{align}

We then consider $\sum_{i=1}^k \sigma_i^2 = \sum_{i=1}^k p_{(i)}(1-p_{(i)}) = \sum_{i=1}^k p_{(i)} - \sum_{i=1}^k p_{(i)}^2$. 
We note that the $p_i$'s (for $i=1, \ldots, n$) are $n$ i.i.d.~random draws from distribution $\dist_n$ whose mean is $\E_{p\sim\dist_n}[p] \ge \frac{1}{2} + \eps_n$, by Hoeffding's inequality, their average satisfies
\[ \frac{1}{n} \sum_{i=1}^n p_i \ge \E_{p\sim \dist_n}[p] - \eps \ge \frac{1}{2} + \eps_n - \eps, \]
with probability at least $1 - \exp(-2n\eps^2)$.  We choose $\eps = O(\sqrt{\frac{\log n} n})$ so the probability is $1 - \frac{1}{n^{\Omega(1)}}$.   
We also note that $\frac{1}{n} \sum_{i=1}^n p_i \le \frac{1}{k} \sum_{i=1}^k p_{(i)}$ because $p_{(1)}, \ldots, p_{(k)}$ are the $k$ largest values in $p_1, \ldots, p_n$.  Therefore,  
\[ \sum_{i=1}^k p_{(i)} \ge \frac{k}{n} \sum_{i=1}^n p_i  \ge k(\frac{1}{2} + \eps_n - \eps). \]
Moreover, since previously we had $p_i \le \frac{1}{2} + \alpha \sqrt{\frac{\log n} n}$ for all $i=1, \ldots, n$, it holds that 
\[ \sum_{i=1}^k p_{(i)}^2 \le k \left(\frac{1}{2} + \alpha \sqrt{\frac{\log n} n}\right)^2 = k(\frac{1}{4} + o(1)). \]
Therefore, 
\[\sum_{i=1}^k \sigma_i^2  = \sum_{i=1}^k p_{(i)} - \sum_{i=1}^k p_{(i)}^2 \ge k(\frac{1}{2} + \eps_n - \eps) - k(\frac{1}{4} + o(1)) = k(\frac{1}{4} - o(1)).\]

Plugging $\sum_{i=1}^k \sigma_i^2  \ge k(\frac{1}{4} - o(1))$ into \Cref{eq:Pr-sigma}, we get
\begin{align*}
    \Pr\left[\sum_{i=1}^k X_{(i)} \le \frac{k}{2} ~\bigg|~ \bm p\right] & \ge \Phi\left(\frac{-\alpha \sqrt{k\log k}}{\sqrt{k(\frac{1}{4} - o(1))}}\right) - \frac{C_1}{\sqrt{k(\frac{1}{4} - o(1))}} \\
    & = \Phi\left(\frac{- 2\alpha \sqrt{\log k}}{1 - o(1)}\right) - \frac{2C_1}{(1 - o(1))\sqrt{k}}.
\end{align*}
Using \Cref{lem:normal-CDF} with $x = \frac{2\alpha \sqrt{\log k}}{1 - o(1)}$, we have
\[ \Phi\left(\frac{- 2\alpha \sqrt{\log k}}{1 - o(1)}\right) \ge \frac{1}{\sqrt{2\pi}} \frac{2\alpha\sqrt{\log k}(1-o(1))}{4\alpha^2 \log k + 1} e^{-\frac{4\alpha^2 \log k}{2(1-o(1))}} = \frac{1}{\sqrt{2\pi}} \frac{1-o(1)}{2\alpha \sqrt{\log k}} \frac{1}{k^{\frac{2\alpha^2}{1-o(1)}}}   \]
which implies 
\[ \Pr\left[\sum_{i=1}^k X_{(i)} \le \frac{k}{2} ~\bigg|~ \bm p\right] \ge  \frac{1}{\sqrt{2\pi}} \frac{1-o(1)}{2\alpha \sqrt{\log k}} \frac{1}{k^{\frac{2\alpha^2}{1-o(1)}}}  - \frac{2C_1}{(1 - o(1))\sqrt{k}},\]
concluding the proof. 
\end{proof}

\begin{proof}[Proof of the second item of \Cref{thm:general-k-comparison}]
To prove $\Gamma^{\bm p}_n(k) < 0$, we use \Cref{lem:Pr-n-upper-bound} and \Cref{lem:Pr-k-lowerbound} to get
\begin{align*}
\Gamma^{\bm p}_n(k) & = \Pr\left[\sum_{i=1}^n X_{(i)} \le \frac{n}{2} ~\bigg|~ \bm p\right] - \Pr\left[\sum_{i=1}^k X_{(i)} \le \frac{k}{2} ~\bigg|~ \bm p\right] \\
& \le \frac{1}{n^{2(a-b)^2}} - \frac{1}{\sqrt{2\pi}}  \frac{1-o(1)}{2\alpha \sqrt{\log k}} \frac{1}{k^{\frac{2\alpha^2}{1-o(1)}}} + \frac{2C_1}{(1-o(1))\sqrt k},
\end{align*}
with probability at least $1 - n^{-2b^2} - n^{-\Omega(1)} = 1 - n^{-\Omega(1)}$, where $\E_{p\sim \dist_n}[p] \ge \frac{1}{2} + \eps_n$ with $\eps_n = a \sqrt{\frac{\log n}{n}}$ for some $a>0$,  $0 < b < a$, $1-F_n(1+\alpha\sqrt{\frac{\log k} k }) = \frac{1}{n^{1+\Omega(1)}}$ for some $\alpha>0$, and $C_1$ is some constant.  Since $k = n^r$, or $n=k^{\frac 1 r}$, 
\begin{align*}
\Gamma^{\bm p}_n(k)
& \le \frac{1}{k^{\frac{2(a-b)^2}{r}}} - \frac{1}{\sqrt{2\pi}}\frac{1-o(1)}{2\alpha \sqrt{\log k}} \frac{1}{k^{\frac{2\alpha^2}{1-o(1)}}} + \frac{2C_1}{(1-o(1))\sqrt k},
\end{align*}
When inequalities $\frac{2\alpha^2}{1-o(1)} < \frac{2(a-b)^2}{r}$ and $\frac{2\alpha^2}{1-o(1)} < \frac{1}{2}$ are satisfied, we have $\Gamma^{\bm p}_n(k) = -O\left(\frac{1}{\sqrt{\log k}}\frac{1}{k^{\frac{2\alpha^2}{1-o(1)}}}\right) < 0$ for sufficiently large $n$.   The former is satisfied when $a > \sqrt{r}\alpha$ and $b$ is sufficiently close to $0$. The latter is satisfied when $\alpha < \frac{1}{2}$. 
\end{proof}

\newpage

\section{Figures}
\label{app:littleOmegaFigures}

\subsection{Optimal Congress Size}
\label{app:opt-congress-simulations}
\begin{figure}[htb]
    \centering
    \includegraphics[width=.99\textwidth]{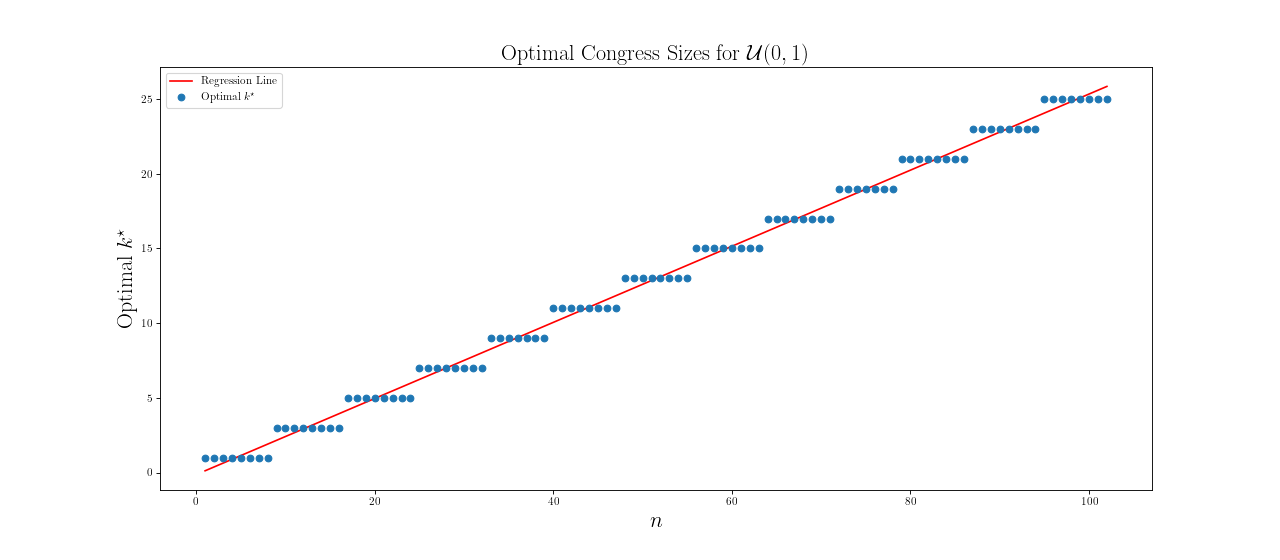}
    \caption{Optimal value of $k$ for $\mathcal{U}(0, 1)$ competence levels following their expectation. The line of best fit is very close to $n/4$.}
    \label{fig:optimal-uniform}
\end{figure}
\newpage
\subsection{Real-world congress sizes}
\label{app:real-world}
\begin{figure}[!hbt]
    \centering
    \includegraphics[width=.99\textwidth]{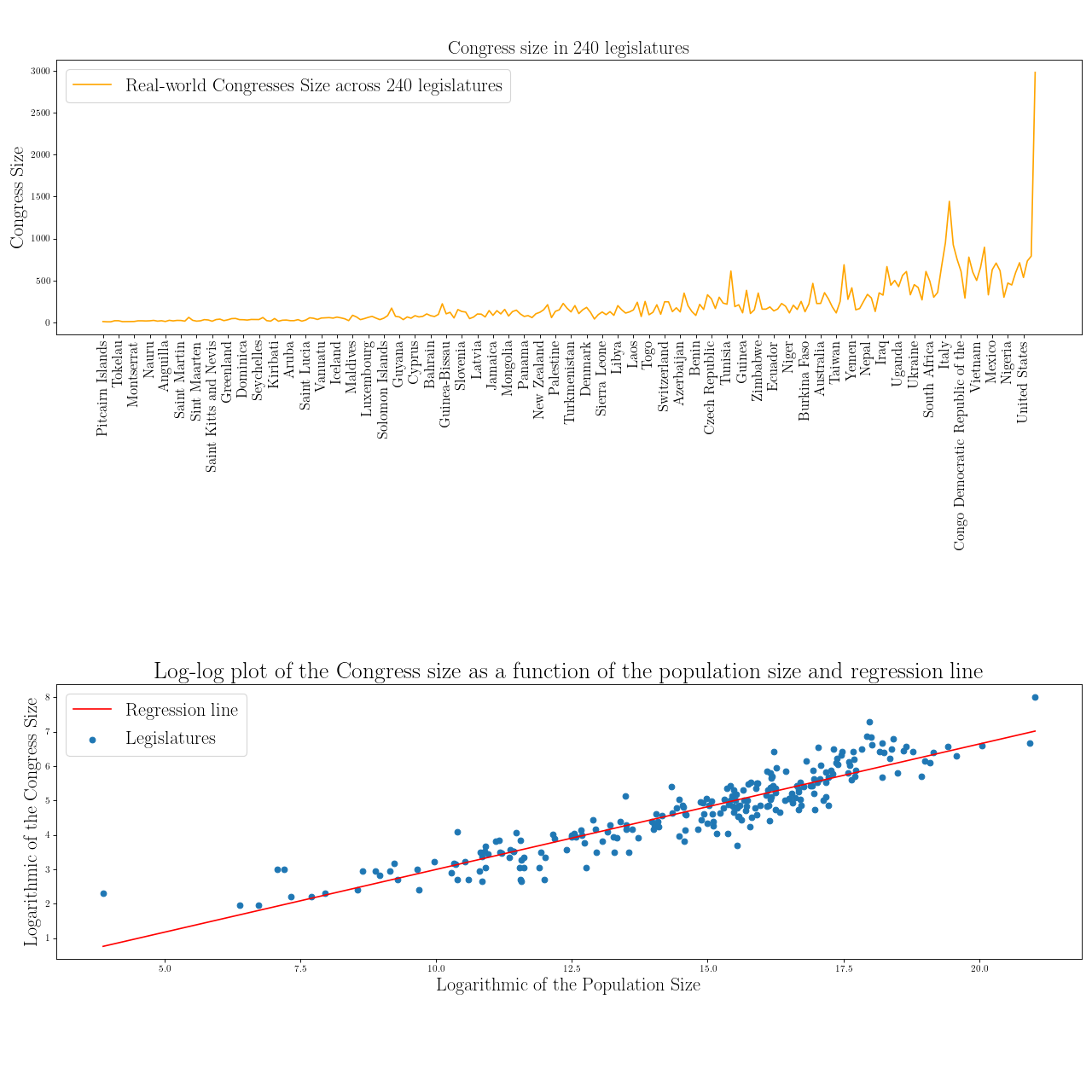}
    \caption{Congress sizes in 240 legislatures (top) and log-log plot of the Congress size as a function of the Population size. The regression line yields $\log k = 0.36 \log n - 0.65$, or $k = c n^{0.36}$, with a coefficient of determination $R^2 = 0.85$.}
    \label{fig:optCongr}
\end{figure}

\subsection{Small congresses outperform majority voting}
\label{app:outperforming-figures}

\begin{figure}[H]
    \centering
    \includegraphics[width=.99\textwidth]{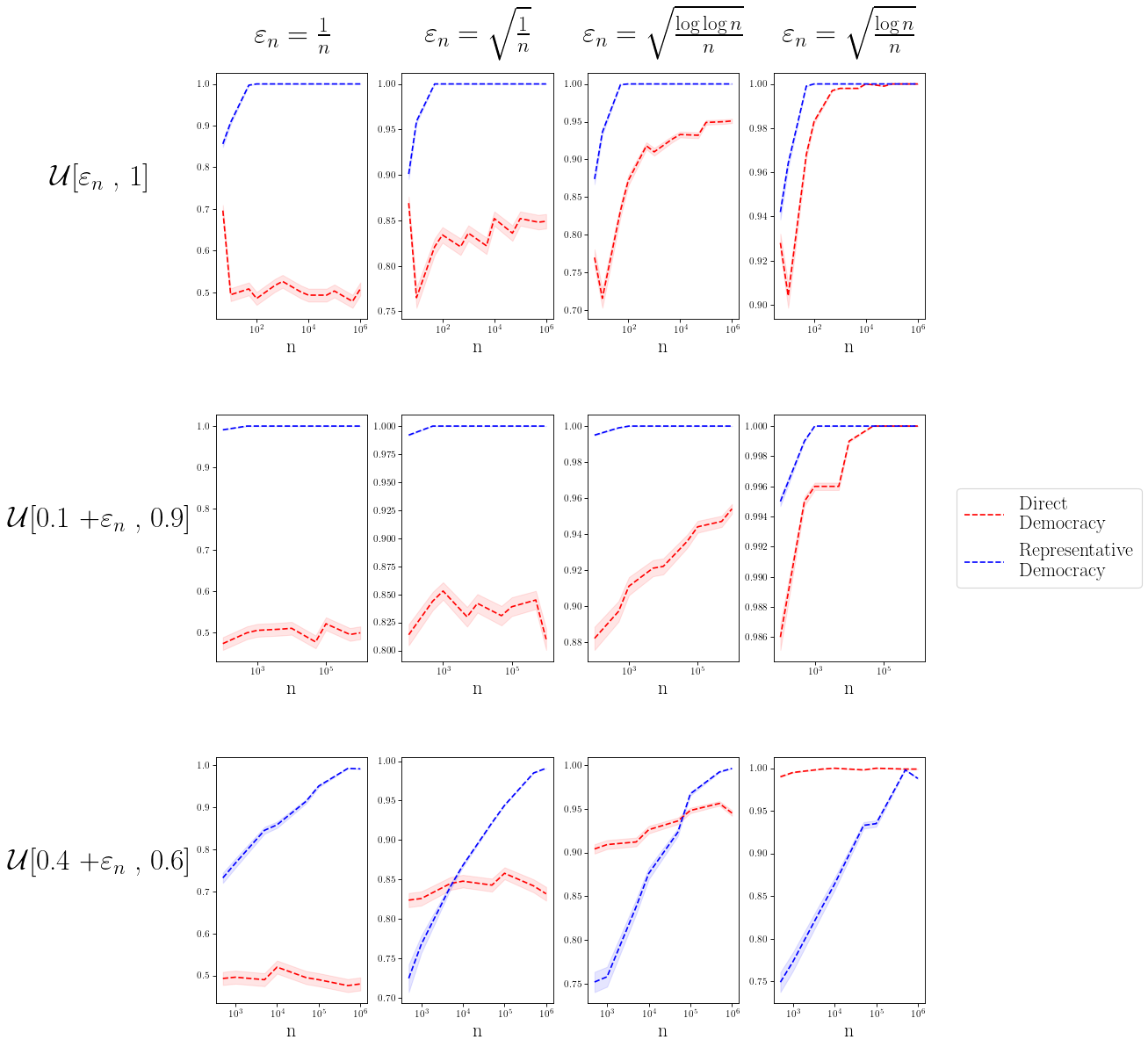}
    \caption{Estimates of $\Pr[\sum_{i=1}^k X_{(i)} > \frac{k}{2}~|~ \bm p]$ (Representative Democracy) and $\Pr[\sum_{i=1}^n X_{(i)} > \frac{n}{2}~|~ \bm p]$ (Direct Democracy) with $95\%$ confidence intervals as a function of the population size for different values of $\eps_n$, with $k=n^{0.36}$ and $\dist_n = \mathcal{U}[0.4+\eps_n, 0.6].$ For large society biases, the population size needs to reach a critical mass for the congress to outperform direct democracy.
    Note that $\mathbb{E}[p_i]=\frac{1 + \eps_n}{2}$ so $\eps_n$ can be thought of as the bias of society towards the correct answer. The top image is for $L=0$, the middle one is for $L=0.1$ and the bottom one for $L=0.4.$
}
    \label{fig:sim}
\end{figure}
Unsurprisingly, the larger the bias, the smaller the gain. 
For $L \leq 0.1$ and a bias of order $\sqrt{\log n / n}$, there is a no gain from relying on the congress, while if the bias is of order $\sqrt{\log\log n / n}$, there is positive gain. Yet, for $L = 0.4$, a bias of order $\sqrt{\log n / n}$ systematically yields a strictly negative gain for $n\leq 10^6$. 

\newpage

\newpage

\section{Distribution Examples}

\subsection{Distributions satisfying \Cref{thm:dictatorship}}
\label{app:dists-satsifying-dictatorship}
We recall the conditions on competency distributions $\dist_n$ under which $\Gamma^{\bm p}_n(1)>0$ or $\Gamma^{\bm p}_n(1)<0$ in \Cref{thm:dictatorship}: 
for $\Gamma^{\bm p}_n(1)>0$, we require $\E_{\dist_n}[p] \le \frac{1}{2} + a \sqrt{\frac{\log n}{n}}$ and $f_n(x) \ge \underline C (1-x)^{\underline \beta - 1}$ for $x\in [1-\underline \delta, 1]$ with constants $a, \underline C, \underline \beta, \underline \delta > 0$ such that $a < \sqrt{ \E_{\dist_n}[p(1-p)] \cdot \min\{1, 2/{\underline{\beta}} \} }$;
for $\Gamma^{\bm p}_n(1)<0$, we require $\E_{\dist_n}[p] \ge \frac{1}{2} + a \sqrt{\frac{\log n}{n}}$ and $f_n(x) \le \overline C$ for $x\in [1-\overline \delta, 1]$ with constants $a, \overline C, \overline \delta > 0$ such that $a > \frac{1}{\sqrt 2}$.  We give examples of beta distributions and uniform distributions satisfying those conditions: 

\begin{example}
\hfill 
\begin{itemize}
    \item Beta distributions: Consider $\dist_n = \mathrm{Beta}(\beta+\eps_n, \beta)$, where $\E_{\dist_n}[p] = \frac{\beta + \eps_n}{2\beta+\eps_n} = \frac{1}{2} + \frac{\eps_n}{4\beta+2\eps_n}$ and $f_n(x) = \frac{1}{\mathrm{B}(\beta + \eps_n, \beta)}x^{\beta+\eps_n-1}(1-x)^{\beta-1}$ where $\mathrm B(\alpha, \beta) = \frac{\Gamma(\alpha) \Gamma(\beta)}{\Gamma(\alpha+\beta)}$.
    Let $\beta$ be a constant and suppose $\eps_n = 4\beta a\sqrt{\frac{\log n}{n}}$.  Since $\eps_n\approx 0$, we have $\E_{\dist_n}[p(1-p)] \approx \frac{1}{4} - \frac{1}{8\beta + 4}$.  
    \begin{itemize}
        \item For $\Gamma^{\bm p}_n(1) > 0$: First, we have $f_n(x) \ge \underline{C}(1-x)^{\beta-1}$ because $\mathrm B(\beta+\eps_n, \beta)$ is upper bounded and $x^{\beta+\eps_n-1}$ is lower bounded for $x$ close to $1$. 
        In addition, $\E_{\dist_n}[p] \le \frac{1}{2} + \frac{\eps_n}{4\beta} = \frac{1}{2} + a\sqrt{\frac{\log n}{n}}$.  When $a < \sqrt{ \E_{\dist_n}[p(1-p)] \cdot \min\{1, 2/{\beta} \} } \approx \sqrt{ (\frac{1}{4} - \frac{1}{8\beta + 4}) \cdot \min\{1, 2/{\beta} \} }$, the condition is satisfied. 
        \item For $\Gamma^{\bm p}_n(1) < 0$: Clearly, $f_n(x) \le \frac{1}{\mathrm{B}(\beta + \eps_n, \beta)} \le \overline{C}<\infty$.
        In addition, $\E_{\dist_n}[p] \approx \frac{1}{2} + \frac{\eps_n}{4\beta} = \frac{1}{2} + a\sqrt{\frac{\log n}{n}}$.  When $a > \frac{1}{\sqrt 2}$, the condition is satisfied. 
    \end{itemize}
    
    \item Uniform distributions: Let $\dist_n=\mathcal U(2\eps_n, 1)$, where $\E_{\dist_n}[p] = \frac{1}{2} + \eps_n$ and $f_n(x) = \frac{1}{1 - 2\eps_n}$.  Let $\eps_n = a\sqrt{\frac{\log n}{n}}$.  Since $\eps_n\approx 0$, we have $\underline{C} = 1 \le f_n(x) \le 2 = \overline{C}$.  Then 
    \begin{itemize}
        \item For $\Gamma^{\bm p}_n(1) > 0$:  the condition is satisfied when $a < \sqrt{ \E_{\dist_n}[p(1-p)] \cdot \min\{1, 2/{\underline{\beta}} \} } \approx \sqrt{ \frac{1}{6} }$ (here $\underline{\beta} = 1$). 
        \item For $\Gamma^{\bm p}_n(1) < 0$: the condition is satisfied when $a > \frac{1}{\sqrt 2}$. 
    \end{itemize}
\end{itemize}
\end{example}

\subsection{Distributions satisfying \Cref{thm:general-k-comparison}}
\label{app:dists-satsifying-general-k}
We recall the conditions on competency distribution $\dist_n$ under which $\Gamma^{\bm p}_n(k) > 0$ or $\Gamma^{\bm p}_n(k) < 0$ in \Cref{thm:general-k-comparison}: for $\Gamma^{\bm p}_n(k) > 0$, we require that its mean satisfies $\E_{\dist_n}[p] \le \frac{1}{2}  + a \sqrt{\frac{\log n}{n}}$ and CDF satisfies $1 - F_n(\frac{1}{2} + \alpha \sqrt{\frac{\log k}{k}}) \ge \frac{k}{n} + \Omega(\sqrt{\frac{\log n}{n}})$ for constants $a, \alpha > 0$ such that $a<\sqrt{\E_{\dist_n}[p(1-p)]}$ and $\alpha > \frac{a}{2\sqrt{r\cdot \E_{\dist_n}[p(1-p)]}}$; for $\Gamma^{\bm p}_n(k) < 0$, we require that its mean satisfies $\E_{\dist_n}[p] \ge \frac{1}{2} + a \sqrt{\frac{\log n}{n}}$ and CDF satisfies $1-F_n(\frac{1}{2} + \alpha\sqrt{\frac{\log k} k }) \le \frac{1}{n^{1+\Omega(1)}}$ for constants $a, \alpha>0$ such that $\alpha < \frac 1 2$ and $a > \sqrt{r} \alpha$.  We give examples of normal distributions and beta distributions satisfying those conditions.  
\begin{example}
Recall that $k=n^r$ for some constant $0<r<1$. 
In this example, we show that distributions with large variance are more likely to satisfy the condition for $\Gamma^{\bm p}_n(k) > 0$ while distributions with small variance satisfy the condition for $\Gamma^{\bm p}_n(k) < 0$.  We consider normal and beta distributions.
\begin{itemize}
    \item Normal distributions: Let $\dist_n$ be the distribution of $p \sim \mathcal N(\mu_n=\frac{1}{2} + a\sqrt{\frac{\log n}{n}}, \sigma_n^2 = \frac{\sigma^2}{k})$ conditioning on $p\in[0, 1]$, where $\sigma^2$ is a constant to be chosen.  We note that for large $k$ (or large $n$), the variance $\sigma_n^2 = \frac{\sigma^2}{k}$ is small, so $p$ is centered around $\mu_n\approx \frac{1}{2}$, thus $\E_{\dist_n}[p(1-p)] \approx \frac{1}{4}$. 
    \begin{itemize}
        \item For $\Gamma^{\bm p}_n(k) > 0$: Let $a, \alpha$ be any constants such that $a < \sqrt{E_{\dist_n}[p(1-p)]} \approx \frac{1}{4}$, $\alpha > \frac{a}{2\sqrt{r\cdot \E_{\dist_n}[p(1-p)]}} \approx \frac{a}{\sqrt r}$. We claim that the CDF condition $1 - F_n(\frac{1}{2} + \alpha \sqrt{\frac{\log k}{k}}) \ge \frac{k}{n} + \Omega(\sqrt{\frac{\log n}{n}})$ is satisfied when $\sigma^2 > \frac{r\alpha^2}{2\min\{1-r, 1/2\}}$. (A proof is given below). 
        \item For $\Gamma^{\bm p}_n(k) < 0$: Let $a, \alpha$ be any constants such that $\alpha < \frac 1 2$ and $a > \sqrt{r} \alpha$. We claim that the CDF condition $1-F_n(\frac{1}{2} + \alpha\sqrt{\frac{\log k} k }) \le \frac{1}{n^{1+\Omega(1)}}$ is satisfied when $\sigma^2 < \frac{r\alpha^2}{2(1+\Omega(1))}$. 
    \end{itemize}
    \item Beta distributions: Let $\dist_n=\mathrm{Beta}(\beta + 4\beta\eps_n, \beta)$ where $\beta = \gamma k$ for some constant $\gamma$ to be chosen, and $\eps_n = a\sqrt{\frac{\log n}n}$.  For simplicity we suppose $r<\frac{1}{2}$, so $\beta\eps_n = \gamma k a \sqrt{\frac{\log n}n} =  \gamma a \frac{\sqrt{\log n}}{n^{1/2-r}} \to 0$ as $n$ grows.  Then the mean satisfies $\E_{\dist_n}[p] = \frac{\beta + 4\beta\eps_n}{2\beta+4\beta\eps_n} = \frac{1}{2} + \frac{\eps_n}{1+2\beta\eps_n} \approx \frac{1}{2} + \eps_n = \frac{1}{2} + a\sqrt{\frac{\log n} n}$.  The variance of $\dist_n=\mathrm{Beta}(\beta + 4\beta\eps_n, \beta)$ is of the order $\frac{1}{8\beta} = \frac{1}{8\gamma k}$, which is larger when $\gamma$ is smaller.  Since the variance is small when $k$ is large, $p\sim\dist_n$ is centered around $\frac{1}{2}$ and hence $\E_{\dist_n}[p(1-p)] \approx \frac{1}{4}$. 
    \begin{itemize}
        \item For $\Gamma^{\bm p}_n(k) > 0$: Let $a, \alpha$ be any constants such that $a < \sqrt{E_{\dist_n}[p(1-p)]} \approx \frac{1}{4}$, $\alpha > \frac{a}{2\sqrt{r\cdot \E_{\dist_n}[p(1-p)]}} \approx \frac{a}{\sqrt r}$. We claim that the CDF condition $1 - F_n(\frac{1}{2} + \alpha \sqrt{\frac{\log k}{k}}) \ge \frac{k}{n} + \Omega(\sqrt{\frac{\log n}{n}})$ is satisfied when $\gamma < \frac{1}{4\alpha^2}\left(\frac{1}{2r} - 1\right)$. 
        \item For $\Gamma^{\bm p}_n(k) < 0$: Let $a, \alpha$ be any constants such that $\alpha < \frac 1 2$ and $a > \sqrt{r} \alpha$. We claim that the CDF condition $1-F_n(\frac{1}{2} + \alpha\sqrt{\frac{\log k} k }) \le \frac{1}{n^{1+\Omega(1)}}$ is satisfied when $\gamma > \frac{1}{4\alpha^2}\left(\frac{1+\Omega(1)}r + 1\right)$. 
    \end{itemize}
\end{itemize}
\end{example}

The rest of this section proves the above claims. 
\begin{proof}[Proof for normal distributions]
Since the random variable $p \sim \mathcal N(\mu_n=\frac{1}{2} + a\sqrt{\frac{\log n}{n}}, \sigma_n^2 = \frac{\sigma^2}{k})$ is below $0$ or above $1$ with exponentially small probability, we can approximate the PDF or CDF of $\dist_n$ by the PDF and CDF of $\mathcal N(\mu_n=\frac{1}{2} + a\sqrt{\frac{\log n}{n}}, \sigma_n^2 = \frac{\sigma^2}{k})$, so
\begin{align*}
    1 - F_n(\frac{1}{2} + \alpha \sqrt{\frac{\log k}{k}}) \approx \int_{\alpha \sqrt{\frac{\log k}{k}}}^\infty \frac{1}{\sqrt{2\pi}\sigma_n} e^{-\frac{(x - \mu_n)^2}{2\sigma_n^2}} \dd x & = \int_{\frac{\alpha \sqrt{\frac{\log k}{k}} - \mu_n}{\sigma_n}}^\infty \frac{1}{\sqrt{2\pi}} e^{-\frac{t^2}{2}} \dd t \\
    & = \int_{\frac{\sqrt k(\alpha \sqrt{\frac{\log k}{k}} - a\sqrt{\frac{\log n}n})}{\sigma}}^\infty \frac{1}{\sqrt{2\pi}} e^{-\frac{t^2}{2}} \dd t \\
    & = \int_{\frac{\alpha}{\sigma} \sqrt{\log k}(1-o(1))}^\infty \frac{1}{\sqrt{2\pi}} e^{-\frac{t^2}{2}} \dd t
\end{align*}
Using $\frac{1}{\sqrt{2\pi}} \frac{x}{x^2+1} e^{-\frac{x^2}{2}} \le \int_x^\infty \frac{1}{\sqrt{2\pi}} e^{-\frac{t^2}{2}}\dd t \le \frac{1}{\sqrt{2\pi}}\frac{1}{x} e^{-\frac{x^2}{2}}$ (\Cref{lem:normal-CDF}), we get
\[
\frac{1}{\sqrt{2\pi}} \frac{\frac{\alpha}{\sigma} \sqrt{\log k}}{(\frac{\alpha}{\sigma} \sqrt{\log k})^2 + 1} e^{-\frac{(\frac{\alpha}{\sigma} \sqrt{\log k})^2}{2}} \le 1 - F_n(\frac{1}{2} + \alpha \sqrt{\frac{\log k}{k}}) \le \frac{1}{\sqrt{2\pi}} \frac{1}{\frac{\alpha}{\sigma} \sqrt{\log k}(1-o(1))} e^{-\frac{(\frac{\alpha}{\sigma} \sqrt{\log k}(1-o(1)))^2}{2}}, 
\]
or asymptotically 
\[
\Omega\left( \frac{1}{\sqrt{\log k}} k^{-\frac{\alpha^2}{2\sigma^2}} \right)\le 1 - F_n(\frac{1}{2} + \alpha \sqrt{\frac{\log k}{k}}) \le O\left( \frac{1}{\sqrt{\log k}} k^{-\frac{\alpha^2}{2\sigma^2}(1-o(1))} \right). 
\]
Plugging in $k=n^r$, 
\[
\Omega\left( \frac{1}{\sqrt{\log n}} \frac{1}{n^{r\frac{\alpha^2}{2\sigma^2}}} \right)\le 1 - F_n(\frac{1}{2} + \alpha \sqrt{\frac{\log k}{k}}) \le O\left( \frac{1}{\sqrt{\log n}} \frac{1}{n^{r\frac{\alpha^2}{2\sigma^2}(1-o(1))}} \right). 
\]

To satisfy the condition for $\Gamma^{\bm p}_n(k) > 0$, it suffices to require
\[ 1- F_n(\frac{1}{2} + \alpha \sqrt{\frac{\log k}{k}}) \ge \Omega\left( \frac{1}{\sqrt{\log n}} \frac{1}{n^{r\frac{\alpha^2}{2\sigma^2}}} \right) \ge \frac{k}{n} + \Omega\left(\sqrt{\frac{\log n} n}\right) = \frac{1}{n^{1-r}} + \Omega\left(\frac{\sqrt{\log n}}{n^{1/2}}\right),\]
which is satisfied when
\[ r\frac{\alpha^2}{2\sigma^2} < \min\{1-r, 1/2\},\]
i.e., $\sigma^2 > \frac{r\alpha^2}{2\min\{1-r, 1/2\}}$.

For $\Gamma^{\bm p}_n(k) < 0$, it suffices to require
\[1 - F_n(\frac{1}{2} + \alpha \sqrt{\frac{\log k}{k}}) \le O\left( \frac{1}{\sqrt{\log n}} \frac{1}{n^{r\frac{\alpha^2}{2\sigma^2}(1-o(1))}} \right) \le \frac{1}{n^{1 + \Omega(1)}},\]
which is satisfied when 
\[ r\frac{\alpha^2}{2\sigma^2} > 1 + \Omega(1),\]
i.e., $\sigma^2 < \frac{r\alpha^2}{2(1+\Omega(1))}$. 
\end{proof}

We then prove the claims for beta distributions. 
\begin{proof}[Proof for beta distributions]
For $\dist_n = \mathrm{Beta}(\beta+4\beta\eps_n, \beta)$, we have
\begin{align}\label{eq:beta-integral}
    1 - F_n(\frac{1}{2} + \alpha \sqrt{\frac{\log k}k}) & = \int_{\frac{1}{2} + \alpha \sqrt{\frac{\log k}k}}^1 \frac{1}{\mathrm B(\beta+4\beta\eps_n, \beta)} x^{\beta+4\beta\eps_n-1}(1-x)^{\beta-1} \dd x\nonumber \\
    & = \int_{\alpha \sqrt{\frac{\log k}k}}^{\frac 1 2} \frac{1}{\mathrm B(\beta+4\beta\eps_n, \beta)} \left(\frac{1}{2} + t\right)^{\beta+4\beta\eps_n-1}\left(\frac{1}{2}-t\right)^{\beta-1} \dd t
\end{align}
where $\mathrm B(\beta+4\beta\eps_n, \beta) = \frac{\Gamma(\beta+4\beta\eps_n)\Gamma(\beta)}{\Gamma(2\beta+4\beta\eps_n)}$, and $\beta = \gamma k$.  We note that since $r<\frac{1}{2}$, $4\beta\eps_n = 4\gamma (n^r) a\frac{\sqrt{\log n}}{n^{1/2}} = o(1) < 1$ as $n$ grows large. 

\textbf{The case of $\Gamma^{\bm p}_n(k) < 0$.}
We first consider the case of $\Gamma^{\bm p}_n(k) < 0$.
We note that by monotonicity of $\Gamma(\cdot)$, assuming $\beta = \gamma k$ is an integer,  
\begin{align*}
    \mathrm B(\beta+4\beta\eps_n, \beta) = \frac{\Gamma(\beta+4\beta\eps_n)\Gamma(\beta)}{\Gamma(2\beta+4\beta\eps_n)} & \ge \frac{\Gamma(\beta)\Gamma(\beta)}{\Gamma(2\beta+1)} \\
    & = \frac{(\beta - 1)!(\beta -1)!}{(2\beta)!} \\
    & = \frac{\beta!\beta!}{(2\beta)!\beta^2}.
\end{align*}
By Stirling's approximation, $\frac{n!n!}{(2n)!} \ge \frac{\sqrt{\pi n}}{4^n}$, hence 
\begin{align*}
    \mathrm B(\beta+4\beta\eps_n, \beta) \ge \frac{\sqrt{\pi \beta}}{4^{\beta}\beta^2}. 
\end{align*}
Plugging into \Cref{eq:beta-integral}, 
\begin{align*}
    1 - F_n(\frac{1}{2} + \alpha \sqrt{\frac{\log k}k}) & \le \int_{\alpha \sqrt{\frac{\log k}k}}^{\frac 1 2} \frac{4^{\beta}\beta^2} {\sqrt{\pi \beta}}\left(\frac{1}{2} + t\right)^{\beta+4\beta\eps_n-1}\left(\frac{1}{2}-t\right)^{\beta-1} \dd t\\
    (\text{because } \frac{1}{2}+t \le 1)~~& \le \int_{\alpha \sqrt{\frac{\log k}k}}^{\frac 1 2} \frac{4^{\beta}\beta^2} {\sqrt{\pi \beta}}\left(\frac{1}{2} + t\right)^{\beta-1}\left(\frac{1}{2}-t\right)^{\beta-1} \dd t\\
    & = \int_{\alpha \sqrt{\frac{\log k}k}}^{\frac 1 2} \frac{4\beta^2}{\sqrt{\pi \beta}}\left(1 + 2t\right)^{\beta-1}\left(1-2t\right)^{\beta-1} \dd t\\
    & = \int_{\alpha \sqrt{\frac{\log k}k}}^{\frac 1 2} \frac{4\beta^2}{\sqrt{\pi \beta}}\left(1 - 4t^2\right)^{\beta-1} \dd t\\
    (\text{using } 1-x \le e^{-x})~~& \le \int_{\alpha \sqrt{\frac{\log k}k}}^{\frac 1 2} \frac{4\beta^2}{\sqrt{\pi \beta}} e^{-4t^2(\beta-1)} \dd t\\
    & \le \int_{\alpha \sqrt{\frac{\log k}k}}^{\frac 1 2} \frac{4e\beta^2}{\sqrt{\pi \beta}} e^{-4t^2\beta} \dd t\\
    (\text{let } u=\sqrt{8\beta}t)~~ & = \int_{\alpha \sqrt{8 \gamma \log k}}^{\frac 1 2 \sqrt{8\gamma k}} \frac{4e\beta}{\sqrt{8\pi}} e^{-\frac{u^2}{2}} \dd u\\
\end{align*}
Using $\int_x^\infty e^{-\frac{u^2}{2}}\dd u \le \frac{1}{x}e^{-\frac{x^2}{2}}$ (\Cref{lem:normal-CDF}), we get 
\begin{align*}
    1 - F_n(\frac{1}{2} + \alpha \sqrt{\frac{\log k}k}) \le \int_{\alpha \sqrt{8 \gamma \log k}}^{\infty} \frac{4e\beta}{\sqrt{8\pi}} e^{-\frac{u^2}{2}} \dd u & \le \frac{4e\beta}{\sqrt{8\pi}} \frac{1}{\alpha \sqrt{8 \gamma \log k}}e^{-\frac{(\alpha \sqrt{8 \gamma \log k})^2}{2}} \\
    & = \frac{e\gamma k}{2\alpha \sqrt{\pi \gamma \log k}} k^{-4\alpha^2\gamma} \\
    & = O\left( \frac{1}{\sqrt{\log k}} \frac{1}{k^{4\alpha^2\gamma-1}}\right) \\
    & = O\left( \frac{1}{\sqrt{\log n}} \frac{1}{n^{r(4\alpha^2\gamma-1)}}\right).
\end{align*}
To satisfy the CDF condition, it suffices to require 
\begin{align*}
    1 - F_n(\frac{1}{2} + \alpha \sqrt{\frac{\log k}k}) \le O\left( \frac{1}{\sqrt{\log n}} \frac{1}{n^{r(4\alpha^2\gamma-1)}}\right) \le \frac{1}{n^{1+\Omega(1)}}, 
\end{align*}
which is satisfied when 
\[ r(4\alpha^2\gamma-1) > 1+\Omega(1),\]
i.e., $\gamma > \frac{1}{4\alpha^2}\left(\frac{1+\Omega(1)}r + 1\right)$.

\textbf{The case of $\Gamma^{\bm p}_n(k) > 0$.}
Now we consider the case of $\Gamma^{\bm p}_n(k) > 0$.
We note that by monotonicity of $\Gamma(\cdot)$, assuming $\beta = \gamma k$ is an integer,  
\begin{align*}
    \mathrm B(\beta+4\beta\eps_n, \beta) = \frac{\Gamma(\beta+4\beta\eps_n)\Gamma(\beta)}{\Gamma(2\beta+4\beta\eps_n)} & \le \frac{\Gamma(\beta+1)\Gamma(\beta)}{\Gamma(2\beta)} \\
    & = \frac{\beta!(\beta -1)!}{(2\beta-1)!} \\
    & = \frac{\beta!\beta!}{(2\beta)!}\frac{2\beta}{\beta}.
\end{align*}
By Stirling's approximation, $\frac{n!n!}{(2n)!} \le \frac{\sqrt{\pi n}}{4^n(1-1/8n)} \le \frac{3}{2}\frac{\sqrt{\pi n}}{4^n}$, hence 
\begin{align*}
    \mathrm B(\beta+4\beta\eps_n, \beta) \le \frac{3\sqrt{\pi \beta}}{4^{\beta}}. 
\end{align*}
Plugging into \Cref{eq:beta-integral}, 
\begin{align*}
    1 - F_n(\frac{1}{2} + \alpha \sqrt{\frac{\log k}k}) & \ge \int_{\alpha \sqrt{\frac{\log k}k}}^{\frac 1 2} \frac{4^{\beta}}{3\sqrt{\pi \beta}}\left(\frac{1}{2} + t\right)^{\beta+4\beta\eps_n-1}\left(\frac{1}{2}-t\right)^{\beta-1} \dd t\\
    (4\beta\eps \le 1)~~& \ge \int_{\alpha \sqrt{\frac{\log k}k}}^{\frac 1 2} \frac{4^{\beta}} {3\sqrt{\pi \beta}}\left(\frac{1}{2} + t\right)^{\beta}\left(\frac{1}{2}-t\right)^{\beta} \dd t\\
    & = \int_{\alpha \sqrt{\frac{\log k}k}}^{\frac 1 2} \frac{1}{3\sqrt{\pi \beta}}\left(1 + 2t\right)^{\beta}\left(1-2t\right)^{\beta} \dd t\\
    & = \int_{\alpha \sqrt{\frac{\log k}k}}^{\frac 1 2} \frac{1}{3\sqrt{\pi \beta}}\left(1 - 4t^2\right)^{\beta} \dd t\\
  (\text{using } (1-\frac{x}{n})^n \ge e^{-x}(1-\frac{x^2}{n}) \text{ for } x\le n)~~  & \ge \int_{\alpha \sqrt{\frac{\log k}k}}^{\frac 1 2} \frac{1}{3\sqrt{\pi \beta}} e^{-4\beta t^2}(1 - 16\beta t^4) \dd t\\
    (\text{let } u=\sqrt{8\beta}t)~~ & = \int_{\alpha \sqrt{8 \gamma \log k}}^{\frac 1 2 \sqrt{8\gamma k}} \frac{1}{3\sqrt{8\pi}\beta} e^{-\frac{u^2}{2}}(1 - \frac{u^4}{4\beta})\dd u\\
    (1 - \frac{u^4}{4\beta} \ge \frac{3}{4} \text{ for } u\le \beta^{1/4})~~& \ge \int_{\alpha \sqrt{8 \gamma \log k}}^{(\gamma k)^{1/4}} \frac{1}{3\sqrt{8\pi}\beta} e^{-\frac{u^2}{2}}\frac{3}{4}\dd u \\
    & = \frac{1}{4\sqrt{8\pi}\beta} \int_{\alpha \sqrt{8 \gamma \log k}}^{(\gamma k)^{1/4}} e^{-\frac{u^2}{2}}\dd u
\end{align*}
Using $\int_x^y e^{-\frac{u^2}{2}}\dd u \ge (-\frac{u}{u^2+1})e^{-\frac{u^2}{2}}\Big|_x^y$ (see the proof of \Cref{lem:normal-CDF}), we get 
\begin{align*}
    1 - F_n(\frac{1}{2} + \alpha \sqrt{\frac{\log k}k}) & \ge \frac{1}{4\sqrt{8\pi}\beta} \int_{\alpha \sqrt{8 \gamma \log k}}^{(\gamma k)^{1/4}} e^{-\frac{u^2}{2}}\dd u \\
    & \ge \frac{1}{4\sqrt{8\pi}\beta} \left( \frac{\alpha\sqrt{8\gamma \log k}}{\alpha^2 8\gamma \log k + 1} e^{-\frac{\alpha^2 8 \gamma \log k}{2}} - \frac{(\gamma k)^{1/4}}{\sqrt{\gamma k} + 1} e^{-\frac{\sqrt{\gamma k}}{2}}\right) \\
    & = \frac{1}{4\sqrt{8\pi}\gamma k} \left( \frac{\alpha\sqrt{8\gamma \log k}}{\alpha^2 8\gamma \log k + 1} k^{-4\alpha^2 \gamma} - o\left(e^{-\frac{\sqrt{\gamma k}}{2}}\right)\right) \\
    & = \Omega\left(\frac{1}{\sqrt{\log k}}\frac{1}{k^{4\alpha^2 \gamma + 1}} \right) \\
    & = \Omega\left(\frac{1}{\sqrt{\log n}}\frac{1}{n^{r(4\alpha^2 \gamma + 1)}} \right) \\
\end{align*}
To satisfy the CDF condition, it suffices to require 
\begin{align*}
    1 - F_n(\frac{1}{2} + \alpha \sqrt{\frac{\log k}k}) \ge \Omega\left(\frac{1}{\sqrt{\log n}}\frac{1}{n^{r(4\alpha^2 \gamma + 1)}} \right) \ge \frac{k}{n} + \Omega\left(\sqrt{\frac{\log n}n}\right) = \frac{1}{n^{1-r}} + \Omega\left(\frac{\sqrt{\log n}}{n^{1/2}}\right), 
\end{align*}
which is satisfied when 
\[ r(4\alpha^2\gamma+1) < \min\{1-r, 1/2\} = 1/2\]
i.e., $\gamma < \frac{1}{4\alpha^2}\left(\frac{1}{2r} - 1\right)$.
\end{proof}

\end{document}

%% file: notations.tex
\DeclareMathOperator*{\argmax}{arg\,max}
\newcommand{\E}{{\mathbb E}}

\newcommand{\set}[1]{{\left\{#1\right\}}}
\newcommand{\eps}{\varepsilon}
\newcommand{\dist}{\mathcal D}
\newcommand{\reals}{\mathbb R}
\newcommand{\dd}{\mathrm d}